\documentclass{article}

\usepackage{microtype}
\usepackage{graphicx}
\usepackage{grffile}
\usepackage{subcaption}
\usepackage{booktabs} % for professional tables

\usepackage{hyperref}

\usepackage{algorithm}
\usepackage{algorithmic}
\usepackage[accepted]{icml2026}

\usepackage{amsmath}
\usepackage{amssymb}
\usepackage{mathtools}
\usepackage{amsthm}
\usepackage{enumitem}
\usepackage{bm}
\usepackage{url}

\newcommand{\conv}{\mathrm{Conv}}
\newcommand{\ffn}{\mathrm{FFN}}
\newcommand{\moe}{\mathrm{MoE}}

\newcommand{\dff}{d_{\mathrm{ff}}}

\newcommand{\Ktilde}{\widetilde{\mathbf{K}}}
\newcommand{\Lbal}{\mathcal{L}_{\mathrm{bal}}}

\newcommand{\Llm}{\mathcal{L}_{\mathrm{LM}}}
\newcommand{\Wdown}{\mathbf{W}_{\mathrm{down}}}
\newcommand{\Wg}{\mathbf{W}_g}
\newcommand{\indicator}{\mathbb{I}}

\usepackage[capitalize,noabbrev]{cleveref}

\theoremstyle{plain}
\newtheorem{theorem}{Theorem}[section]
\newtheorem{proposition}[theorem]{Proposition}

\newtheorem{corollary}[theorem]{Corollary}
\theoremstyle{definition}

\newtheorem{assumption}[theorem]{Assumption}
\theoremstyle{remark}
\newtheorem{remark}[theorem]{Remark}

\usepackage[textsize=tiny]{todonotes}

\icmltitlerunning{cMoLLM at Scale: Horizontal Scaling Laws for Mixture-of-LLMs}

\begin{document}

\twocolumn[
  \icmltitle{cMoLLM at Scale: Horizontal Scaling Laws for\\
    Convolutionally-Gated Mixture-of-LLMs}

  % It is OKAY to include author information, even for blind submissions: the
  % style file will automatically remove it for you unless you've provided
  % the [accepted] option to the icml2026 package.

  \icmlsetsymbol{equal}{*}
  \icmlsetsymbol{cores}{$\dagger$}

  \begin{icmlauthorlist}
    \icmlauthor{Xin Yang}{zju}
    \icmlauthor{Yemin Wang}{xmu}
    \icmlauthor{Mingda Liu}{ist}
    \icmlauthor{Letian Li}{tbsi}
    \icmlauthor{Shuaishuai Cao}{csu}
    \icmlauthor{Zhengxiao He}{ua}
    \icmlauthor{Ryan Dong}{ind,cores}
  \end{icmlauthorlist}

  \icmlaffiliation{zju}{School of Mathematical Sciences, Zhejiang University}
  \icmlaffiliation{tbsi}{Shenzhen International Graduate School, Tsinghua University}
  \icmlaffiliation{xmu}{Xiamen University}
  \icmlaffiliation{ind}{Independent Researcher}
  \icmlaffiliation{ua}{Department of Electrical and Computer Engineering, University of Arizona}
  \icmlaffiliation{ist}{Institute of Science Tokyo}
  \icmlaffiliation{csu}{School of Central South University}

  \icmlcorrespondingauthor{Ryan Dong}{zwdong618@gmail.com}

  \icmlkeywords{Mixture-of-LLMs, Pipeline-Level Scaling, Dynamic Convolution, Convolutionally-Gated Routing, Large Language Models}

  \vskip 0.3in
]

\printAffiliationsAndNotice{}
\begin{figure*}[!t]
  \centering
  \includegraphics[width=1\textwidth]{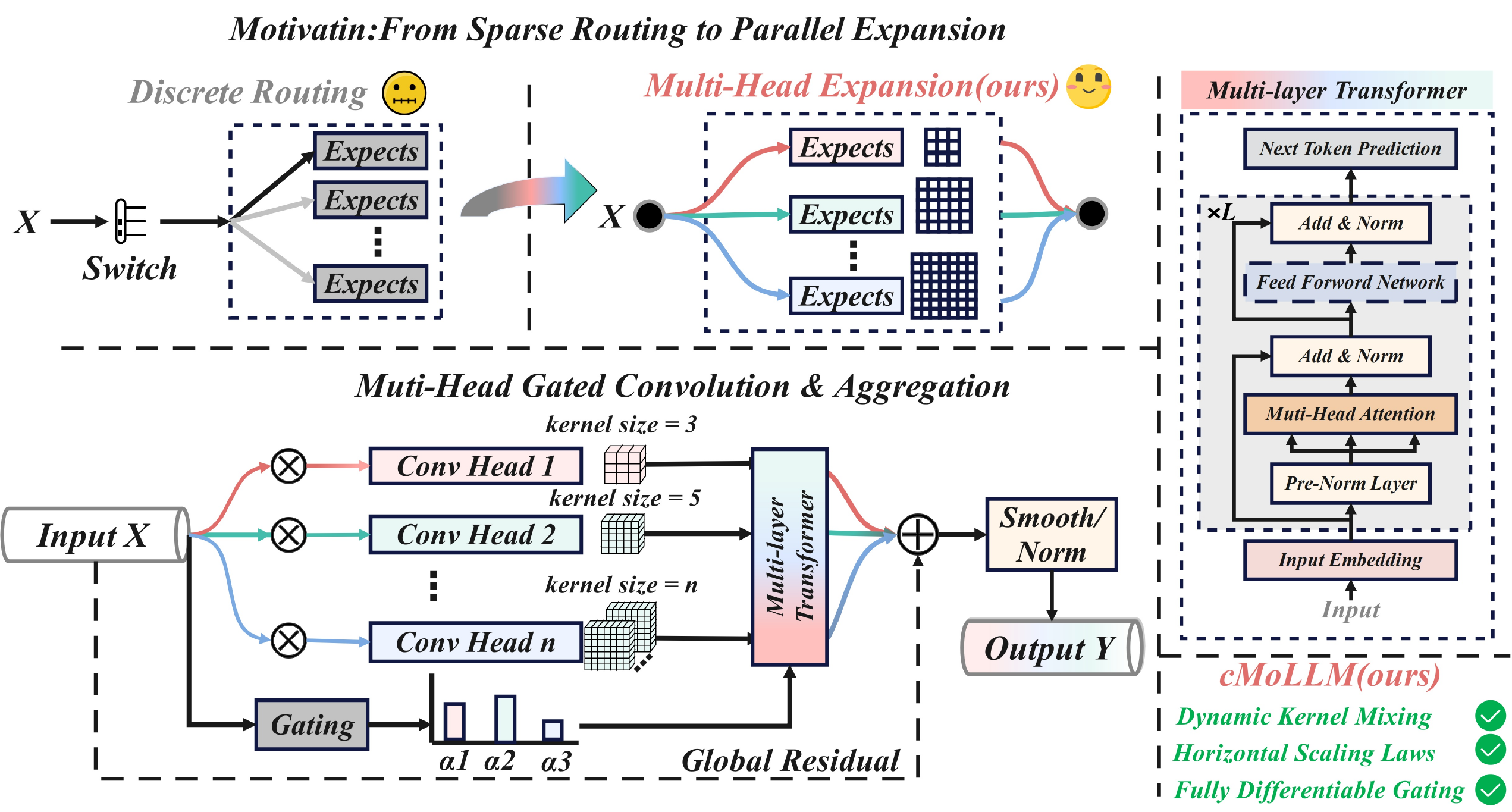}
  \caption{\textbf{Teaser of cMoLLM.} cMoLLM scales capacity at the pipeline level via convolutionally-gated mixture over end-to-end streams, yielding better perplexity and downstream accuracy than dense baselines under matched compute, without virtual tokens or auxiliary heads.}
  \label{fig:teaser}
\end{figure*}
%%%%%%%%%%%%%%%%%%%%%%%%%%%%%%%%%%%%%%%%%%%%%%%%%%%%%%%%%%%%%%%%%%%%%%%%%%%%%%%
% ABSTRACT
%%%%%%%%%%%%%%%%%%%%%%%%%%%%%%%%%%%%%%%%%%%%%%%%%%%%%%%%%%%%%%%%%%%%%%%%%%%%%%%
\begin{abstract}
Scaling large language models (LLMs) has driven their success, yet dense Transformers couple capacity and computation: every parameter is activated for every token, making training and inference costs grow linearly with model size—a critical bottleneck as models approach trillion-parameter regimes.
We aim to scale capacity through MoE-style mixture throughout the LLM pipeline rather than only the FFN. Prior pipeline-level approaches include ParaScale\cite{chen2025parallel}, which introduces virtual tokens and parallel streams but incurs substantial overhead and suffers from homogenized routing and gradient collapse, and AltUp\cite{baykal2023alternating}, which uses an auxiliary prediction branch but offers limited adaptivity and slow convergence.
We establish that MoE-style mixture layers can be reformulated as variable-kernel dynamic convolutions, where each expert corresponds to a $1{\times}1$ convolutional kernel and routing implements input-conditioned kernel aggregation. Building on this equivalence, we introduce cMoLLM: a convolutionally gated mixture-of-LLMs that routes over end-to-end streams through fully differentiable dynamic convolution.
In GPT-2-style models trained on FineWeb, cMoLLM improves language modeling perplexity and downstream GLUE and SQuAD accuracy under matched compute, with better stream utilization, more stable optimization, and favorable scaling compared to ParaScale- and AltUp-style baselines.
\end{abstract}
\vspace{-0.4cm}

\begin{table*}[!t]
  \renewcommand{\arraystretch}{1.1}
  \caption{Systematic qualitative comparison of pipeline-level capacity scaling methods: ParaScale, AltUp, and cMoLLM.}
  \label{tab:qualitative_compare}
  \vskip 0.1in
  \begin{center}
  \begin{normalsize}
  \begin{tabular}{@{}lccccc@{}}
    \toprule
    \textbf{Method} & \textbf{Extra tok.} & \textbf{Aux.\ branch} & \textbf{Routing} & \textbf{Compute} & \textbf{Stability} \\
    \midrule
    ParaScale & Yes & No  & Parallel streams & Higher & Collapse risk \\
    AltUp     & No  & Yes & Fixed pred.\ branch & Low & Slow convergence \\
    cMoLLM & No  & No  & Dynamic conv. & $\sim$dense & Stable \\
    \bottomrule
  \end{tabular}
  \end{normalsize}
  \end{center}
  \vskip -0.1in
\end{table*}
%%%%%%%%%%%%%%%%%%%%%%%%%%%%%%%%%%%%%%%%%%%%%%%%%%%%%%%%%%%%%%%%%%%%%%%%%%%%%%%
% 1. INTRODUCTION
%%%%%%%%%%%%%%%%%%%%%%%%%%%%%%%%%%%%%%%%%%%%%%%%%%%%%%%%%%%%%%%%%%%%%%%%%%%%%%%
\section{Introduction}
\label{sec:intro}

The remarkable success of large language models (LLMs)~\cite{gpt41_2025,claude4_2025,DBLP:conf/smc/YangTWJJ25} is closely tied to scale: increasing model capacity consistently improves performance across diverse tasks~\citep{radford2019language,brown2020language,kaplan2020scaling,hoffmann2022training,DBLP:journals/corr/abs-2603-03314,DBLP:journals/corr/abs-2510-13291}.
Training and inference cost grow roughly linearly with model size, as every parameter is activated for every token.
As models approach trillion-parameter regimes, this linear compute--capacity coupling has become a dominant bottleneck, motivating the search for more parameter-efficient ways to expand capacity.

A natural direction is to use a Mixture-of-Experts (MoE)--style approach to scale capacity: route tokens across multiple sub-models and combine their outputs.
Most prior MoE work applies this only to the \textbf{feed-forward network (FFN)}: experts are FFN blocks, and routing is typically discrete (e.g., Top-$K$), which leads to expert collapse, skewed utilization, and brittle training~\citep{shazeer2017outrageously,fedus2022switch}.
We instead pursue pipeline-level mixture: routing across entire LLM pipelines (or stream-wise sub-models) rather than FFN-only experts.

Two works are directly relevant.
\textbf{AltUp}~\citep{baykal2023alternating} widens token representations and uses a virtual prediction branch to update inactive blocks; it increases effective capacity with small parameter overhead but relies on a fixed, hand-designed structure that is not fully adaptive and often converges slowly.
This increases sequence length and thus compute; moreover, routing can homogenize across streams, leading to gradient collapse and training instability.
We seek a new pipeline-level scaling mechanism that keeps the MoE-style routing across sub-models intuition, avoids virtual tokens, auxiliary prediction heads, and the pitfalls above.

Our starting point is a theoretical insight: MoE-style mixture layers can be exactly rewritten as dynamic convolutions with input-dependent (variable) kernels.
Formally, each expert corresponds to a $1{\times}1$ convolutional kernel; the router performs input-dependent mixing of these kernels (\cref{thm:equiv}).
This establishes MoE--dynamic convolution equivalence and gives a unified lens for analyzing sparse, conditional computation.
We then approximate this ideal: each ``stream'' uses a distinct convolutional kernel, and we mix streams via soft, fully differentiable gating---no Top-$K$, no low-rank factorization.

Building on this, we propose \textbf{cMoLLM}: a convolutionally-gated mixture-of-LLMs that applies conditional computation to the entire Transformer pipeline (\cref{fig:teaser}).
We maintain a small set of end-to-end ``streams,'' each associated with its own $1{\times}1$ kernel; a lightweight gating network produces input-dependent mixture weights, and the mixed kernel is applied via standard (grouped) pointwise convolution.
All streams are combined through a stable, fully differentiable dynamic convolution---no virtual tokens, no auxiliary heads, no Top-$K$ or low-rank---yielding parameter-efficient pipeline-level capacity scaling with hardware-friendly convolutional primitives.

Our contributions are as follows:
\begin{itemize}[leftmargin=*,itemsep=2pt,topsep=2pt]
  \item \textbf{Theoretical:} We establish the formal equivalence between MoE-style mixture layers and dynamic convolutions with variable kernels (\cref{thm:equiv}). This provides a unified theoretical framework for analyzing and designing sparse, conditional computation models, beyond FFN-level MoE.
  \item \textbf{Methodological:} We introduce \textbf{cMoLLM}, a pipeline-level convolutionally-gated mixture whose kernel-sharing and stream structure follow directly from the convolution view. It achieves parameter-efficient scaling and training stability without virtual tokens, auxiliary prediction branches, Top-$K$ routing, or low-rank factorization.
  \item \textbf{Empirical:} On GPT-2--style models trained on FineWeb, cMoLLM matches or improves perplexity, GLUE, and SQuAD over dense baselines under comparable cost, with better stream utilization, more stable training dynamics, and favorable scaling (see \cref{sec:scaling_analysis}) compared to ParaScale- and AltUp-style pipeline mixtures.
\end{itemize}

\begin{figure*}[!t]
  \vskip 0.1in
  \begin{center}
    \includegraphics[width=2.0\columnwidth]{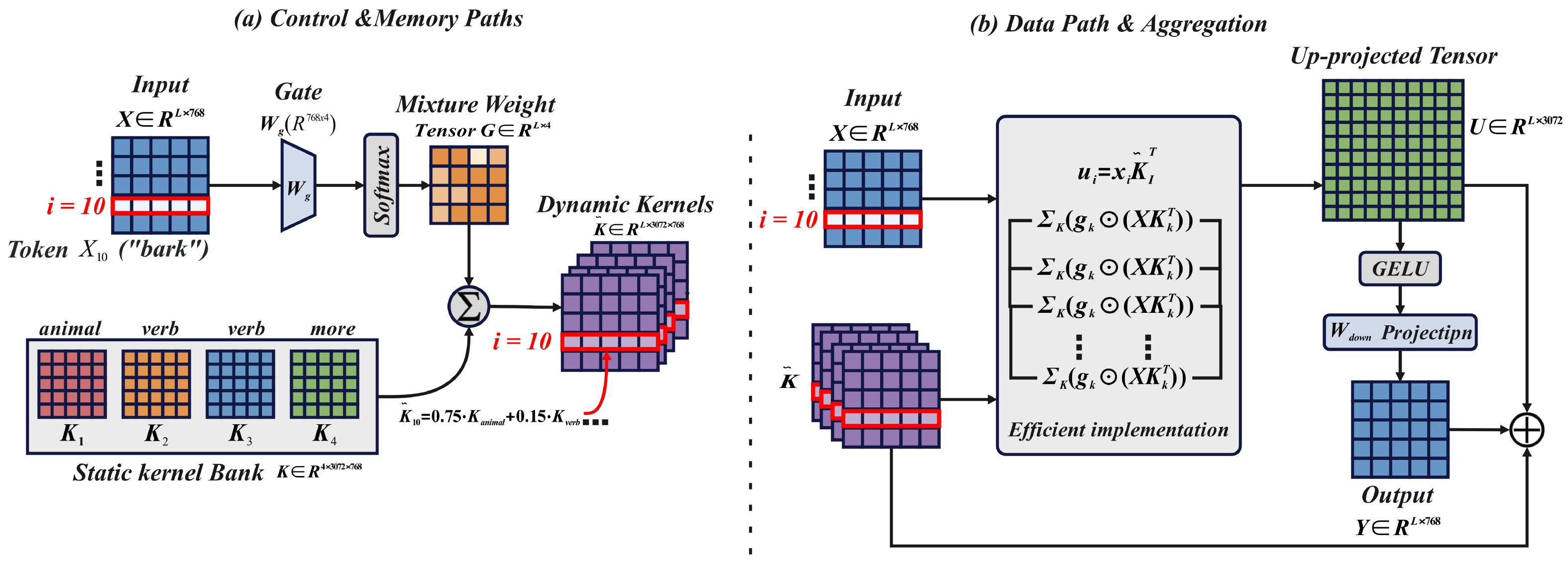}
    \caption{\textbf{Architecture (Fig.~2).} cMoLLM block: input $\mathbf{X}$ passes through a gating network to produce soft mixture weights $\{g_k\}$; each stream has a $1{\times}1$ kernel $\mathbf{K}_k$; the effective kernel $\widetilde{\mathbf{K}}(\mathbf{x}) = \sum_k g_k(\mathbf{x})\mathbf{K}_k$ is applied via grouped $1{\times}1$ convolution. No virtual tokens, no Top-$K$, no auxiliary heads.}
    \label{fig:cMoLLM}
    \vspace{-0.3cm}
  \end{center}
  \vskip -0.15in
\end{figure*}
%%%%%%%%%%%%%%%%%%%%%%%%%%%%%%%%%%%%%%%%%%%%%%%%%%%%%%%%%%%%%%%%%%%%%%%%%%%%%%%
% 2. RELATED WORK
%%%%%%%%%%%%%%%%%%%%%%%%%%%%%%%%%%%%%%%%%%%%%%%%%%%%%%%%%%%%%%%%%%%%%%%%%%%%%%%
\section{Related Work}
\label{sec:related}

\noindent\textbf{Mixture-of-Experts and Pipeline-Level Scaling.}
Classical MoE approaches~\citep{shazeer2017outrageously,fedus2022switch,lepikhin2020gshard} sparsify the \textbf{feed-forward network (FFN)} layer using Top-$K$ routing, where each token is routed to a sparse subset of expert FFN blocks.
Large-scale MoE systems such as GLaM~\citep{du2022glam}, BASE Layers~\citep{lewis2021base}, DeepSpeed-MoE~\citep{rajbhandari2022deepspeed}, and Vision MoE models~\citep{riquelme2021scaling} focus primarily on FFN-level sparsity and system-level optimizations for training and inference.
In contrast, we target pipeline-level mixture over entire LLM streams, leveraging the MoE--dynamic convolution equivalence (\cref{sec:equiv}) to design convolutionally-gated routing without Top-$K$ truncation.
The most directly relevant prior work includes \textbf{ParaScale}~\citep{chen2025parallel}, which scales capacity via parallel virtual token streams but increases compute and suffers from gradient collapse, and \textbf{AltUp}~\citep{baykal2023alternating}, which uses a hand-designed auxiliary prediction branch but converges slowly.
\Cref{tab:qualitative_compare} provides a systematic comparison; cMoLLM avoids virtual tokens and auxiliary heads, using fully differentiable dynamic convolution for stable, parameter-efficient scaling.
% \vspace{-0.1cm}

\noindent\textbf{Dynamic Convolution and Conditional Computation.}
Dynamic convolution~\citep{jia2016dynamic,yang2019condconv,chen2020dynamic,wu2019pay,ma2020weightnet,zhang2020dynet} adapts convolution kernels based on input, typically within CNN backbones or sequence models.
More broadly, conditional computation and routing networks~\citep{bengio2013estimating,rosenbaum2017routing,mcgill2017deciding} learn to select computation paths depending on inputs.
Our key theoretical contribution is establishing the formal equivalence between MoE-style mixture layers and dynamic convolutions with variable kernels (\cref{thm:equiv}), which enables us to design a pipeline-level, convolutionally-gated mixture over LLM streams with explicit probabilistic routing and load-balancing objectives.

\noindent\textbf{Parameter-Efficient Methods and Surveys.}
Parameter-efficient fine-tuning methods such as adapters~\citep{houlsby2019parameter} and LoRA~\citep{hu2022lora} add or reparameterize a small set of trainable weights on top of frozen backbones, primarily targeting efficient fine-tuning.
cMoLLM instead modifies the pretraining architecture to improve the capacity--compute trade-off; these approaches are complementary and could in principle be combined.
Comprehensive surveys of MoE models and their applications in LLMs and beyond are provided by \citet{cai2025survey,mu2025comprehensive,liu2024survey}, which situate our contribution within the broader MoE landscape.

%%%%%%%%%%%%%%%%%%%%%%%%%%%%%%%%%%%%%%%%%%%%%%%%%%%%%%%%%%%%%%%%%%%%%%%%%%%%%%%
% 3. PRELIMINARIES
%%%%%%%%%%%%%%%%%%%%%%%%%%%%%%%%%%%%%%%%%%%%%%%%%%%%%%%%%%%%%%%%%%%%%%%%%%%%%%%
\section{Preliminaries}
\label{sec:prelim}

We first introduce the notation used throughout the paper and briefly review the core concepts that underpin our formulation.

\subsection{Notation and Assumptions}
\label{sec:notation}

We use bold lowercase letters (e.g., $\mathbf{x}$) for vectors, bold uppercase letters (e.g., $\mathbf{W}$) for matrices, and calligraphic letters (e.g., $\mathcal{S}$) for sets.
The hidden dimension of the Transformer is $d$, the intermediate FFN dimension is $\dff$, the sequence length is $L$ and the number of experts/streams is $N$.
For $\mathbf{X} \in \mathbb{R}^{L \times d}$, the $i$-th token is denoted by $\mathbf{x}_i$.

Throughout, we make the following mild assumptions.

\begin{assumption}[Bounded Inputs]
\label{ass:bounded}
There exists $R > 0$ such that $\|\mathbf{x}\|_2 \le R$ for all token representations $\mathbf{x}$ encountered during training and evaluation.
\end{assumption}

\begin{assumption}[Lipschitz Nonlinearity]
\label{ass:lipschitz}
The activation function $\sigma$ is $L_\sigma$-Lipschitz and satisfies $\sigma(0)=0$ (e.g., ReLU or GELU).
\end{assumption}

These conditions are standard in theoretical analyses of deep networks and MoE
architectures, and they suffice for establishing the complexity and stability
results in this paper.

\subsection{Transformer Feed-Forward Network}
A standard Transformer FFN applies two linear projections with a nonlinearity~\citep{vaswani2017attention}:
\begin{equation}
  \ffn(\mathbf{x}) = \mathbf{W}_2 \, \sigma(\mathbf{W}_1 \mathbf{x} + \mathbf{b}_1) + \mathbf{b}_2,
  \label{eq:ffn}
\end{equation}
where $\mathbf{x} \in \mathbb{R}^d$ is the input token representation, $\mathbf{W}_1 \in \mathbb{R}^{\dff \times d}$, $\mathbf{W}_2 \in \mathbb{R}^{d \times \dff}$, and $\sigma$ is typically GELU or ReLU. The FFN is applied independently to each token.

\subsection{Mixture-of-Experts Layer}

An MoE layer replaces a single feed-forward network with a collection of $N$
experts $\{E_k\}_{k=1}^N$ and a gating function $G$ that routes each input to a
subset of experts~\citep{shazeer2017outrageously}:
\begin{equation}
  \moe(\mathbf{x})
  = \sum_{k=1}^{N} g_k(\mathbf{x})\,E_k(\mathbf{x}),
  \label{eq:moe}
\end{equation}
where $g(\mathbf{x})\in\mathbb{R}^N$ denotes the routing weights produced by the
gate. In practice, only the top-$K$ entries of $g(\mathbf{x})$ are nonzero,
yielding sparse computation. Each expert $E_k$ is typically a two-layer FFN with
its own parameters.

\subsection{Pointwise Convolution as a Linear Transform}

A pointwise ($1{\times}1$) convolution applies the same linear map across token
positions. Given an input sequence $\mathbf{X}\in\mathbb{R}^{L\times d_{\mathrm{in}}}$
and a kernel $\mathbf{K}\in\mathbb{R}^{d_{\mathrm{out}}\times d_{\mathrm{in}}}$,
the operation is given by
\begin{equation}
  \conv_{1{\times}1}(\mathbf{X}; \mathbf{K})
  = \mathbf{X}\mathbf{K}^\top,
  \label{eq:conv1x1}
\end{equation}
which corresponds to applying an identical token-wise linear projection from
$\mathbb{R}^{d_{\mathrm{in}}}$ to $\mathbb{R}^{d_{\mathrm{out}}}$.
This observation serves as a key building block for our subsequent analysis.

%%%%%%%%%%%%%%%%%%%%%%%%%%%%%%%%%%%%%%%%%%%%%%%%%%%%%%%%%%%%%%%%%%%%%%%%%%%%%%%
% 4. METHOD: cMoLLM
%%%%%%%%%%%%%%%%%%%%%%%%%%%%%%%%%%%%%%%%%%%%%%%%%%%%%%%%%%%%%%%%%%%%%%%%%%%%%%%
\section{Method: cMoLLM}
\label{sec:method}

We now present our main theoretical result and the cMoLLM architecture.

\subsection{MoE as Dynamic Convolution: A Formal Equivalence}
\label{sec:equiv}

We show that MoE layers admit an equivalent interpretation as dynamic pointwise
convolutions. In this subsection we focus on linear experts and the structure of the router; extensions to nonlinear experts are discussed in \cref{app:proof} and summarized in \cref{cor:nonlinear}.

\begin{theorem}[MoE--Dynamic Convolution Equivalence]
\label{thm:equiv}
Let $\mathbf{x}\in\mathbb{R}^{d}$ denote a token representation.
Consider an MoE layer with $N$ linear experts
$E_k(\mathbf{x})=\mathbf{x}\mathbf{W}_k^\top$, where
$\mathbf{W}_k\in\mathbb{R}^{d_{\mathrm{out}}\times d}$,
and routing weights $\{g_k(\mathbf{x})\}_{k=1}^N$ satisfying
$\sum_{k=1}^N g_k(\mathbf{x})=1$ for all $\mathbf{x}$.
Then the MoE output can be written as a dynamic $1{\times}1$ convolution:
\begin{equation}
  \moe(\mathbf{x})
  = \conv_{1{\times}1}\!\Bigl(
  \mathbf{x};\, \tilde{\mathbf{K}}(\mathbf{x})
  \Bigr)
  = \mathbf{x}\,\tilde{\mathbf{K}}(\mathbf{x})^\top,
  \label{eq:equiv}
\end{equation}
where the effective kernel is given by
\begin{equation}
  \tilde{\mathbf{K}}(\mathbf{x})
  := \sum_{k=1}^{N} g_k(\mathbf{x})\,\mathbf{W}_k.
\end{equation}
Since $\tilde{\mathbf{K}}(\mathbf{x})$ depends on the input $\mathbf{x}$, the
MoE layer is precisely a form of dynamic convolution.
\end{theorem}

\begin{proof}
By linearity of matrix multiplication:
\begin{align}
\sum_{k=1}^{N} g_k(\mathbf{x}) \cdot (\mathbf{x} \mathbf{W}_k^\top) 
&= \mathbf{x} \Bigl( \textstyle\sum_{k=1}^{N} g_k(\mathbf{x}) \mathbf{W}_k \Bigr)^{\!\top} \nonumber \\
&= \conv_{1{\times}1}\!\Bigl(\mathbf{x};\, \textstyle\sum_{k=1}^{N} g_k(\mathbf{x}) \mathbf{W}_k \Bigr).
\end{align}
The weighted sum of expert weight matrices forms a dynamic kernel that varies with $\mathbf{x}$. (Classical MoE uses Top-$K$ routing so only $K$ terms are nonzero; we use soft mixing over all streams.) See \cref{app:proof} for extension to nonlinear experts.
\renewcommand{\qedsymbol}{}
\end{proof}

\begin{corollary}[Nonlinear Two-Layer Experts]
\label{cor:nonlinear}
Under \cref{ass:lipschitz}, consider two-layer experts of the form
$E_k(\mathbf{x}) = \sigma(\mathbf{x} \mathbf{W}_{k,1}^\top)\mathbf{W}_{k,2}^\top$
with routing weights $\{g_k(\mathbf{x})\}_{k=1}^N$ satisfying $\sum_k g_k(\mathbf{x}) = 1$.
Then the MoE layer can still be written as a dynamic $1{\times}1$ convolution
with an input-dependent effective kernel that incorporates a data-dependent
mask induced by $\sigma$; see \cref{app:proof} for details.
\end{corollary}

Assumptions~\ref{ass:bounded} and~\ref{ass:lipschitz} are not needed for \cref{thm:equiv} itself, but they will be used in our later complexity, stability, and toy-model analyses that build on this equivalence.

\begin{remark}[Key Differences from Standard Convolution]
\label{rem:diff}
While \cref{thm:equiv} reveals structural similarity, standard convolutional layers differ from explicit MoE in three aspects:
\begin{enumerate}[leftmargin=*,itemsep=2pt]
  \item \textbf{Static vs.\ Learnable Gating.} In CNN, activation functions (e.g., ReLU) act as fixed, non-learnable gates determined by the sign of pre-activations. In MoE, the router is a trainable subnetwork whose routing policy is learned end-to-end.
  \item \textbf{Unnormalized vs.\ Normalized Weights.} CNN activations are unnormalized; MoE gating weights satisfy $\sum_k g_k = 1$, representing a probability distribution over experts.
 \item \textbf{Dense vs.\ Sparse Computation.}
Standard CNN layers apply all filters densely; classical MoE routes each token to only the top-$K$ experts. We use soft mixture weights over all streams (no Top-$K$), so computation scales with $N$ streams but remains a single dynamic convolution.
\end{enumerate}
\end{remark}

This equivalence motivates our approach: we implement pipeline-level mixture as an explicit dynamic convolution, leveraging efficient convolutional primitives.

\subsection{cMoLLM Block Design}
\label{sec:cMoLLM_block}

Based on \cref{thm:equiv}, we design \textbf{cMoLLM} as a pipeline-level mixture. \Cref{fig:cMoLLM} illustrates the architecture: gating, per-stream kernels, and dynamic convolution.

\begin{figure*}[!t]
  \centering
  \includegraphics[width=1\textwidth]{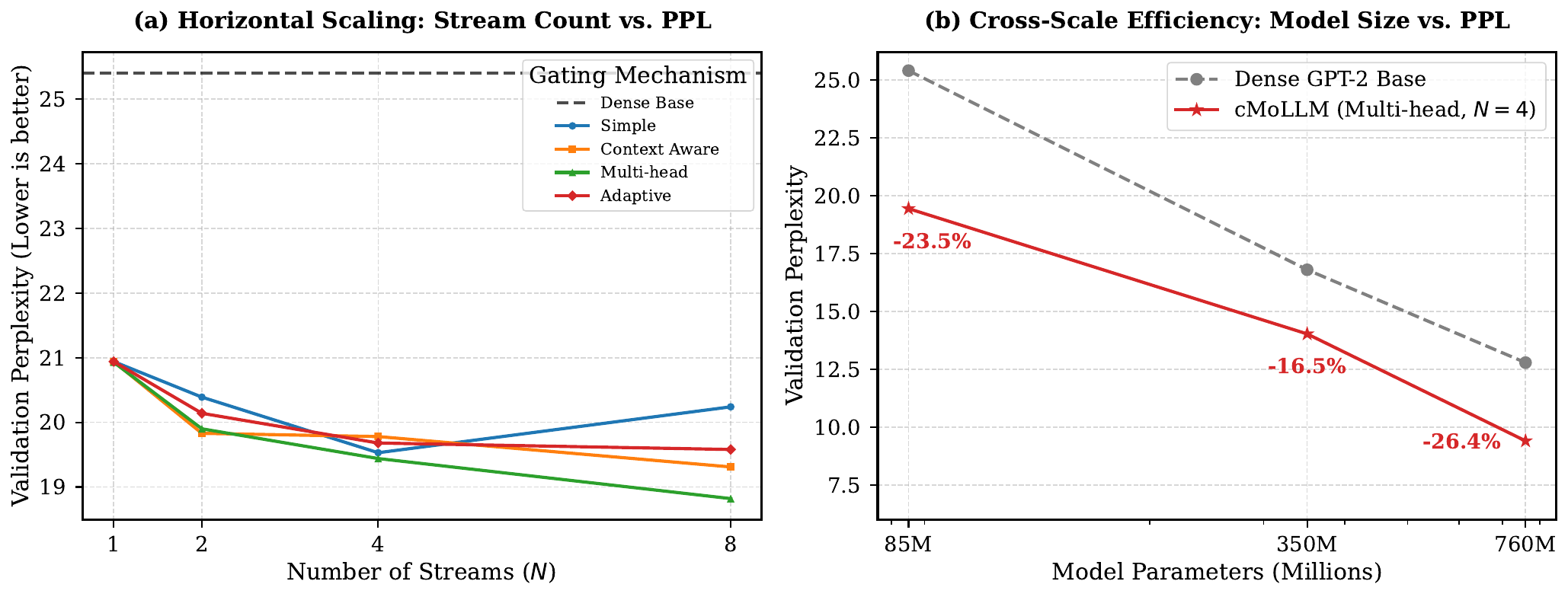}
\caption{\textbf{Experimental results (Fig.~3).} Left: per-stream gating weight distribution across layers. Right: validation loss and perplexity vs.\ stream count $n$ or training steps. cMoLLM maintains stable utilization and scaling gains.}
  \label{fig:exp_figures}
\end{figure*}

\noindent\textbf{Gating Network.}
Given input $\mathbf{X} \in \mathbb{R}^{L \times d}$, a lightweight MLP produces routing logits:
\begin{equation}
\begin{aligned}
  % \mathbf{H} = \mathbf{X} \Wg + \mathbf{b}_g \in \mathbb{R}^{L \times N},
  \mathbf{H}_{\text{flat}} = \text{Flatten}(\mathbf{Y}_{\text{streams}}) \in \mathbb{R}^{L \times (N \cdot d)},\\
  \mathbf{H}_{\text{mid}} = \sigma\bigl(\mathbf{H}_{\text{flat}} \mathbf{W}_g^{(1)} + \mathbf{b}_g^{(1)}\bigr) \in \mathbb{R}^{L \times d},\\
  \mathbf{H}_{\text{logits}} = \mathbf{H}_{\text{mid}} \mathbf{W}_g^{(2)} + \mathbf{b}_g^{(2)} \in \mathbb{R}^{L \times N},
\end{aligned}
\end{equation}
where $\Wg \in \mathbb{R}^{d \times N}$. We apply a softmax over streams to obtain mixture weights:
\begin{equation}
  g_k(\mathbf{x}) = \mathrm{softmax}(\mathbf{h})_k, \quad k \in [N].
  \label{eq:gating}
\end{equation}
This fully differentiable gating avoids discrete Top-$K$ truncation; all streams participate via soft weights.
We support multiple gating variants (\texttt{simple}, \texttt{context\_aware}, \texttt{multi\_head}, \texttt{adaptive}), evaluated in \cref{sec:experiments}.

\noindent\textbf{Expert Kernel Generator.}
Each stream $k$ has its own $1{\times}1$ convolutional kernel $\mathbf{K}_k \in \mathbb{R}^{\dff \times d}$. The effective kernel is their mixture weighted by $g_k(\mathbf{x})$; we use no low-rank factorization or shared base in our main setup.

\noindent\textbf{Dynamic Grouped Convolution.}
Let a set of mixture weights \(\{g_k(\mathbf{x})\}_{k=1}^N\) be given. We form a dynamic mixture of convolutional kernels by
\begin{equation}
\tilde{\mathbf{K}}(\mathbf{x}) = \sum_{k=1}^N g_k(\mathbf{x})\, \mathbf{K}_k .
\end{equation}
Practically, we compute branch-wise outputs in parallel and fuse them:
\begin{equation}
\mathbf{X}_k = \mathrm{Conv}_k(\mathbf{X}; \mathbf{K}_k), \quad k \in [N],
\end{equation}
and concatenate the per-branch outputs along the representation dimension:
\begin{equation}
\mathbf{X}_{\text{input}} = [\mathbf{X}_1, \mathbf{X}_2, \ldots, \mathbf{X}_N] \in \mathbb{R}^{L \times N \times d},
\end{equation}
where \(L\) is the sequence length and \(d\) is the feature dimension (hidden size).

To model interactions among the branches, we apply a transformer block to the stream of parallel outputs:
\begin{equation}
\mathbf{Y}_{\mathrm{stream}} = \mathrm{TRANSFORMERS}\bigl(\mathbf{X}_{\text{input}};\; \tilde{\mathbf{K}}(\mathbf{X})\bigr).
\end{equation}
Here \(\mathrm{TRANSFORMERS}(\cdot)\) denotes one or more Transformer layers, and the attention (or feed-forward) computations may be conditioned on the dynamic kernels \(\tilde{\mathbf{K}}(\mathbf{X})\).

\noindent\textbf{ParaScale Output Aggregation.}
Let the hidden states have shape \(\mathbf{H} \in \mathbb{R}^{S \times B \times H}\). If the ParaScale module uses \(N_{\mathrm{ps}} > 1\) parallel scales, we first reorganize the hidden states along the parallel dimension:
\begin{equation}
\mathbf{H} \rightarrow \hat{\mathbf{H}} = \operatorname{reshape}\bigl(\mathbf{H},\, S,\, B/N_{\mathrm{ps}},\, N_{\mathrm{ps}},\, H\bigr),
\end{equation}
and flatten the parallel axis:
\begin{equation}
\mathbf{H}_{\text{flat}} = \operatorname{reshape}\bigl(\hat{\mathbf{H}},\, S,\, B/N_{\mathrm{ps}},\, N_{\mathrm{ps}}\,H\bigr).
\end{equation}
We compute an aggregation weight via a learned layer:
\begin{equation}
\boldsymbol{\alpha}_{s,b} = \operatorname{softmax}\bigl(\mathbf{W}_{\mathrm{agg}}\, \mathbf{H}_{\text{flat}}^{s,b,*} \bigr) \in \mathbb{R}^{N_{\mathrm{ps}}}, \quad \|\boldsymbol{\alpha}_{s,b}\|_{1} = 1 .
\end{equation}

\noindent\textbf{Smoothness and Gradient-Aware Measures.}
1) Per-step attention smoothing. Let \(\boldsymbol{\alpha}_{s,b}\) be the per-step, per-sample attention over the \(N_{\mathrm{ps}}\) parallel scales. We define the smoothed attention as
\begin{equation}
\boldsymbol{\alpha}_{s,b}^{\,\text{smooth}} = (1-\beta)\,\boldsymbol{\alpha}_{s,b} + \beta\,\frac{1}{N_{\mathrm{ps}}}\,\mathbf{1}, \qquad
\beta \in [0,1].
\label{eq:alpha_smooth}
\end{equation}

2) Gradient-aware aggregation. The weighted sum uses the smoothed weights:
\begin{equation}
\mathbf{H}_{\mathrm{weighted}}^{s,b} = \sum_{i=1}^{N_{\mathrm{ps}}} \alpha_{s,b}^{(i)\,\text{smooth}}\, \mathbf{H}_{\text{flat}}^{s,b,i},
\label{eq:weighted}
\end{equation}
followed by a linear projection and a mean residual term:
\begin{equation}
\mathbf{H}_{\mathrm{agg}}^{s,b} = \mathbf{F}\,\mathbf{H}_{\mathrm{weighted}}^{s,b} + \mathbf{H}_{\text{mean}}^{s,b}, \quad
\mathbf{H}_{\mathrm{out}}^{s,b} = \mathrm{Proj}\bigl(\mathbf{H}_{\mathrm{agg}}^{s,b}\bigr).
\label{eq:agg_out}
\end{equation}

\subsection{cMoLLM Forward Pass: Pseudocode}
\label{sec:pseudocode}

\Cref{alg:cMoLLM} gives pseudocode for a single Transformer layer with a cMoLLM block (no Top-$K$, no low-rank).

\begin{algorithm}[t]
  \caption{Forward pass of a Transformer layer with cMoLLM}
  \label{alg:cMoLLM}
  \begin{algorithmic}[1]
    \STATE \textbf{Input:} $\mathbf{X} \in \mathbb{R}^{L \times d}$, number of streams $N$, kernels $\{\mathbf{K}_k\}_{k=1}^N$, $\Wdown$
    \STATE \textbf{Self-Attention:} $\mathbf{H} \leftarrow \mathrm{SelfAttention}(\mathbf{X})$
    \STATE \textbf{Gating logits:} $\mathbf{G} \leftarrow \mathbf{H} \Wg + \mathbf{b}_g \in \mathbb{R}^{L \times N}$
    \STATE \textbf{Soft mixture weights:} $g_{i,k} \leftarrow \mathrm{softmax}(\mathbf{G}_{i,:})_k$ for each token $i$, stream $k$
    \STATE \textbf{Mixed kernel:} $\Ktilde_i \leftarrow \sum_{k=1}^{N} g_{i,k} \, \mathbf{K}_k$
    \STATE \textbf{Dynamic $1{\times}1$ conv:} $\mathbf{U}_i \leftarrow \mathbf{H}_i \, \Ktilde_i^\top$
    \STATE \textbf{Nonlinearity \& down-proj:} $\mathbf{Y}_i \leftarrow \Wdown \, \sigma(\mathbf{U}_i)$
    \STATE \textbf{Output:} $\mathbf{Y}$ (residual + norm as usual)
  \end{algorithmic}
\end{algorithm}

\subsection{CNN as a Constrained Implicit MoE}
\label{sec:cnn_moe}

The equivalence in \cref{thm:equiv} suggests that standard convolutional layers can be interpreted as special cases of MoE under strong constraints.
We formalize this intuition for a single linear or convolutional layer.

\begin{proposition}[CNN as Implicit Constrained MoE]
\label{prop:cnn_moe}
Consider a linear layer $\mathbf{y} = \mathbf{W}\mathbf{x}$ with $\mathbf{W} \in \mathbb{R}^{d_{\mathrm{out}}\times d}$ and activation $\sigma$ satisfying \cref{ass:lipschitz}.
Let $\mathbf{w}_j^\top$ denote the $j$-th row of $\mathbf{W}$ and define per-output ``experts''
$E_j(\mathbf{x}) = \sigma(\mathbf{w}_j^\top \mathbf{x})$ for $j \in [d_{\mathrm{out}}]$.
Then the layer can be written as an MoE of $d_{\mathrm{out}}$ experts with:
\begin{enumerate}[leftmargin=*,itemsep=2pt]
  \item static, parameter-free gating determined solely by the sign pattern of pre-activations;
  \item unnormalized gating weights given by $\sigma(\mathbf{w}_j^\top \mathbf{x})$;
  \item dense computation, as all experts are evaluated for every input.
\end{enumerate}
In particular, a standard CNN layer implements an implicit, highly constrained MoE over its output channels.
\end{proposition}

\begin{proof}
Write the layer output as $\mathbf{y} = \sigma(\mathbf{W}\mathbf{x})$, where $\sigma$ is applied element-wise.
For each output dimension $j$, we have $y_j = \sigma(\mathbf{w}_j^\top \mathbf{x})$, which we interpret as the output of expert $E_j$.
Define the gating weight $g_j(\mathbf{x}) = 1$ for all $j$ and inputs.
Then the overall output can be written as
\begin{equation}
  \mathbf{y}
  = \sum_{j=1}^{d_{\mathrm{out}}} g_j(\mathbf{x}) \, E_j(\mathbf{x})
  = \sum_{j=1}^{d_{\mathrm{out}}} E_j(\mathbf{x}),
\end{equation}
which matches the standard layer.
The nonlinearity $\sigma$ induces a data-dependent binary mask $\indicator(\mathbf{w}_j^\top \mathbf{x} > 0)$ on each expert, yielding a form of fixed, unnormalized gating as discussed in \cref{rem:diff}.
Unlike explicit MoE, all experts are always evaluated, so computation is dense.
\renewcommand{\qedsymbol}{}
\end{proof}

This proposition connects classical CNNs to MoE: they sit at one end of a spectrum with static, dense, and unnormalized routing, whereas cMoLLM uses learned, normalized routing via dynamic convolutions (soft mixture over all streams, no Top-$K$).
Viewed through this lens, standard CNNs are implicit, constrained MoE models, and \textbf{cMoLLM} can be seen as relaxing these constraints---moving along a continuum from fixed, dense routing to learned, normalized, and capacity-controllable routing over pipeline-level streams.

\subsection{Gating Variants (Convolution Types)}
\label{sec:conv_type}

We implement four gating variants corresponding to the \texttt{simple}, \texttt{context\_aware}, \texttt{multi\_head}, and \texttt{adaptive} options evaluated in \cref{sec:experiments}.

\subsubsection{Simple Gated Convolution}

Two parallel convolutions are applied: one for feature extraction and one for gate generation. The feature output is element-wise multiplied by the sigmoid-activated gate output. Gating is local and token-wise, with no global context.

\subsubsection{Context-Aware Gated Convolution}

A context analyzer (global average pooling plus an MLP with 4:1 compression) produces channel-wise modulation weights from the full sequence; these are combined with the local gate convolution output. Gating thus depends on both local patterns and sequence-level statistics.

\subsubsection{Multi-Head Gated Convolution}

Multiple independent gate convolutions operate on the same input; each head produces gates for a disjoint subset of output channels. Outputs are concatenated and fused by a learned layer. Heads can specialize in different temporal scales or aspects of the input.

\subsubsection{Adaptive Gated Convolution}

An adaptive network (global pooling plus a two-layer MLP) outputs three bounded parameters (sensitivity, bias, magnitude) from global input statistics. The gate is sigmoid(scaled gate-conv output + bias) times magnitude. Gating behavior is adjusted per sequence (e.g., dynamic range or complexity).

\subsection{Training Strategies}
\label{sec:training}

\noindent\textbf{Load Balancing Loss.}
To encourage balanced use of streams, we add an auxiliary loss~\citep{fedus2022switch}:
\begin{equation}
  \Lbal = \alpha \cdot N \cdot \sum_{k=1}^{N} p_k^2,
  \label{eq:load_balance}
\end{equation}
where $p_k = \frac{1}{L}\sum_{i=1}^{L} g_k(\mathbf{x}_i)$ is the average gating probability for stream $k$ over the sequence. Minimizing $\sum_k p_k^2$ discourages collapse onto a few streams. We use $\alpha = 0.01$ in experiments.

\noindent\textbf{Total Loss.}
The total training objective is $\mathcal{L} = \Llm + \Lbal$; we use no kernel regularization, progressive sparsification, or low-rank terms.

\subsection{Computational Complexity Analysis}
\label{sec:complexity}

We compare per-token cost of dense FFN, Top-$K$ MoE, and cMoLLM. Let $\mathrm{FLOPs}(\cdot)$ denote leading-order FLOPs per token.

Dense FFN: $\mathrm{FLOPs}_{\mathrm{dense}} \approx 2 d \dff$ (up- and down-projections).
Top-$K$ MoE with $N$ experts: $\mathrm{FLOPs}_{\mathrm{moe}} \approx 2 K d \dff + \mathrm{FLOPs}_{\mathrm{gate}}$.

In cMoLLM we compute $\mathbf{U}_i = \sum_{k=1}^{N} g_{i,k} (\mathbf{H}_i \mathbf{K}_k^\top)$ then $\Wdown \sigma(\mathbf{U}_i)$. The $N$ stream products cost $N \cdot d \cdot \dff$, and the down-projection $\dff \cdot d$. Thus
\begin{equation}
  \mathrm{FLOPs}_{\mathrm{cMoLLM}} \approx (N + 1) d \dff + \mathrm{FLOPs}_{\mathrm{gate}}.
  \label{eq:flops_cMoLLM}
\end{equation}
For small $N$ (e.g., $N{=}4$), this is on the same order as dense, while affording $N$ distinct kernels and pipeline-level mixture. Convolution primitives often yield favorable throughput in practice.

Combining the above, we can summarize the horizontal scaling behavior as follows.
Fix a compute budget $\mathrm{FLOPs}_0$ and FFN dimension $\dff$.
Then there exists a constant $C$ (absorbing gating overhead) such that
\begin{equation}
  \mathrm{FLOPs}_{\mathrm{cMoLLM}} \le C \cdot \mathrm{FLOPs}_{\mathrm{dense}}
  \quad \text{whenever } N \le C - 1,
\end{equation}
while the number of distinct kernels (and thus the effective capacity of the mixture) grows linearly with $N$.
Top-$K$ MoE trades compute for sparsity:
increasing $N$ at fixed $K$ primarily increases parameter count but leaves per-token compute approximately $O(K d \dff)$.
cMoLLM therefore realizes a horizontal scaling law: for bounded $N$, we scale capacity roughly linearly in $N$ while keeping per-token compute within a constant factor of the dense baseline.

\begin{table*}[t]
  \renewcommand{\arraystretch}{1.1}
  \caption{Full experimental results (3 seeds; mean $\pm$ std): validation loss, perplexity (PPL), GLUE (\%), SQuAD v2 (\%). ``conv'' = gating variant; $n$ = streams. \textbf{Bold} = best (SOTA) in column.}
  \label{tab:main_results}
  \vskip 0.1in
  \begin{center}
  \begin{normalsize}
  \begin{tabular}{@{}lcccccc@{}}
    \toprule
    \textbf{conv} & $n$ & \textbf{Loss} $\downarrow$ & \textbf{PPL} $\downarrow$ & \textbf{GLUE (\%)} $\uparrow$ & \textbf{SQuAD v2 (\%)} $\uparrow$ \\
    \midrule
    base & -- & $3.23{\pm}0.03$ & $25.40{\pm}0.65$ & $42.45{\pm}0.99$ & $50.07{\pm}0.65$ \\
    \midrule
    \texttt{simple} & 1 & $3.04{\pm}0.03$ & $20.94{\pm}0.22$ & $44.16{\pm}0.56$ & $51.13{\pm}0.54$ \\
    \texttt{simple} & 2 & $3.01{\pm}0.02$ & $20.39{\pm}0.56$ & $44.62{\pm}1.02$ & $52.08{\pm}0.95$ \\
    \texttt{simple} & 4 & $2.97{\pm}0.04$ & $19.53{\pm}0.49$ & $44.92{\pm}0.87$ & $55.71{\pm}1.30$ \\
    \texttt{simple} & 8 & $3.00{\pm}0.03$ & $20.24{\pm}0.12$ & $44.68{\pm}0.33$ & $55.95{\pm}0.51$ \\
    \midrule
    \texttt{context\_aware} & 1 & $3.04{\pm}0.03$ & $20.94{\pm}0.41$ & $44.16{\pm}0.42$ & $51.12{\pm}1.28$ \\
    \texttt{context\_aware} & 2 & $2.98{\pm}0.03$ & $19.83{\pm}0.55$ & $46.79{\pm}1.21$ & $52.13{\pm}1.13$ \\
    \texttt{context\_aware} & 4 & $2.98{\pm}0.02$ & $19.78{\pm}0.12$ & $47.12{\pm}0.25$ & $55.64{\pm}0.87$ \\
    \texttt{context\_aware} & 8 & $2.96{\pm}0.01$ & $19.31{\pm}0.68$ & $\mathbf{48.98{\pm}0.64}$ & $56.11{\pm}0.73$ \\
    \midrule
    \texttt{multi\_head} & 1 & $3.04{\pm}0.01$ & $20.93{\pm}0.22$ & $44.16{\pm}1.18$ & $51.18{\pm}0.66$ \\
    \texttt{multi\_head} & 2 & $2.99{\pm}0.05$ & $19.90{\pm}0.28$ & $47.78{\pm}0.45$ & $52.98{\pm}1.35$ \\
    \texttt{multi\_head} & 4 & $2.96{\pm}0.04$ & $19.44{\pm}0.41$ & $48.02{\pm}0.79$ & $56.37{\pm}1.56$ \\
    \texttt{multi\_head} & 8 & $\mathbf{2.93{\pm}0.02}$ & $\mathbf{18.82{\pm}0.37}$ & $48.82{\pm}1.01$ & $\mathbf{57.06{\pm}0.52}$ \\
    \midrule
    \texttt{adaptive} & 1 & $3.04{\pm}0.05$ & $20.94{\pm}0.46$ & $44.16{\pm}0.98$ & $51.28{\pm}0.88$ \\
    \texttt{adaptive} & 2 & $3.00{\pm}0.02$ & $20.14{\pm}0.26$ & $46.64{\pm}0.43$ & $52.43{\pm}1.45$ \\
    \texttt{adaptive} & 4 & $2.97{\pm}0.03$ & $19.68{\pm}0.71$ & $47.03{\pm}0.36$ & $55.89{\pm}1.39$ \\
    \texttt{adaptive} & 8 & $2.97{\pm}0.03$ & $19.58{\pm}0.65$ & $47.81{\pm}0.75$ & $56.74{\pm}0.72$ \\
    \bottomrule
  \end{tabular}
  \end{normalsize}
  \end{center}
  \vskip -0.1in
\end{table*}
% \vspace{-0.2cm}

\begin{table}[t]
  \renewcommand{\arraystretch}{1.2}
  \caption{Scaling experiments (3 seeds; mean $\pm$ std) across model sizes (\cref{tab:model_config}). \texttt{mh} = multi-head cMoLLM; \texttt{base} = dense baseline. \textbf{Bold} = SOTA in column.}
  \label{tab:scaling_results}
  \vskip 0.1in
  \begin{center}
  \begin{tiny}
  \begin{tabular}{@{}lcccc@{}}
    \toprule
    \textbf{type} & \textbf{Loss} $\downarrow$ & \textbf{PPL} $\downarrow$ & \textbf{GLUE} $\uparrow$ & \textbf{SQuAD v2} $\uparrow$ \\
    \midrule
    \texttt{base-85M} & $3.23{\pm}0.03$ & $25.40{\pm}0.65$ & $42.45{\pm}0.99$ & $50.07{\pm}0.65$ \\
    \texttt{base-350M} & $2.82{\pm}0.04$ & $16.80{\pm}0.47$ & $50.68{\pm}1.08$ & $59.24{\pm}1.24$ \\
    \texttt{base-760M} & $2.55{\pm}0.05$ & $12.79{\pm}0.28$ & $55.74{\pm}0.46$ & $62.91{\pm}0.74$ \\
    \midrule
    \texttt{mh-85M} & $2.96{\pm}0.03$ & $19.44{\pm}0.56$ & $48.02{\pm}0.84$ & $56.37{\pm}0.54$ \\
    \texttt{mh-350M} & $2.64{\pm}0.03$ & $14.03{\pm}0.37$ & $53.82{\pm}1.03$ & $61.08{\pm}1.26$ \\
    \texttt{mh-760M} & $\mathbf{2.24{\pm}0.02}$ & $\mathbf{9.41{\pm}0.45}$ & $\mathbf{58.45{\pm}0.77}$ & $\mathbf{64.56{\pm}1.32}$ \\
    \bottomrule
  \end{tabular}
  \end{tiny}
  \end{center}
  \vskip -0.1in
\end{table}

\subsection{Intuitive Summary}
\label{sec:intuition}

We summarize cMoLLM intuitively.
MoE-style mixture linearly combines expert outputs; \cref{thm:equiv} shows this is equivalent to mixing $1{\times}1$ convolution kernels---i.e., dynamic convolution with variable kernels.
We approximate this with per-stream kernels $\mathbf{K}_k$ and soft mixture weights $g_k(\mathbf{x})$; no Top-$K$, no low-rank.

cMoLLM is implemented with standard $1{\times}1$ convolutions and a lightweight gating network.
Recipe: (i) keep the Transformer backbone; (ii) replace FFN blocks with cMoLLM blocks (streams + gating); (iii) tune $N$ to trade capacity vs.\ compute.
Experiments (\cref{sec:experiments}) show better perplexity and GLUE at similar compute, with more stable stream utilization than ParaScale- and AltUp-style pipeline mixtures.
In addition, a cluster-structured toy model in \cref{app:cluster} formalizes when routing (and by equivalence, dynamic convolution) can achieve Bayes-optimal performance while any single linear classifier suffers a nontrivial error floor, illustrating the potential benefits of cMoLLM-style conditional computation.

%%%%%%%%%%%%%%%%%%%%%%%%%%%%%%%%%%%%%%%%%%%%%%%%%%%%%%%%%%%%%%%%%%%%%%%%%%%%%%%
% 5. EXPERIMENTS
%%%%%%%%%%%%%%%%%%%%%%%%%%%%%%%%%%%%%%%%%%%%%%%%%%%%%%%%%%%%%%%%%%%%%%%%%%%%%%%

\begin{table}[t]
  \renewcommand{\arraystretch}{1.1}
  \caption{Model scales used in scaling experiments. Config 1 = small (85M), Config 2 = medium (350M), Config 3 = large (760M).}
  \label{tab:model_config}
  % \vskip 0.1in
  \begin{center}
  \begin{normalsize}
  \begin{tabular}{@{}ccccc@{}}
    \toprule
    \textbf{Config} & \textbf{Layers} & \textbf{Hidden} & \textbf{Heads} & \textbf{Params} \\
    \midrule
    1 & 12 & 768 & 12 & $\sim$85M \\
    2 & 24 & 1024 & 16 & $\sim$350M \\
    3 & 24 & 1536 & 24 & $\sim$760M \\
    \bottomrule
  \end{tabular}
  \end{normalsize}
  \end{center}
\end{table}

\begin{table}[t]
  \renewcommand{\arraystretch}{1.1}
  \caption{Shared training hyperparameters across all variants.}
  \label{tab:config}
  \begin{center}
  \begin{normalsize}
  \begin{tabular}{@{}lc@{}}
    \toprule
    \textbf{Hyperparameter} & \textbf{Value} \\
    \midrule
    Base architecture & GPT-2 \\
    FFN dimension ($\dff$) & 3072 \\
    Sequence length & 4096 \\
    Learning rate & $6 \times 10^{-5}$ \\
    \bottomrule
  \end{tabular}
  \end{normalsize}
  \end{center}
\end{table}

\section{Experiments}
\label{sec:experiments}

We evaluate cMoLLM on language modeling, comparing to a dense GPT-2 baseline under matched training setups.

\subsection{Experimental Setup}
\label{sec:setup}

\noindent\textbf{Model Configuration.}
We adopt a GPT-2-style architecture~\citep{radford2019language} as our base model. We evaluate at three scales (\cref{tab:model_config}): small (~85M), medium (~350M), and large (~760M) parameters, varying layers, hidden size, and attention heads. \Cref{tab:config} gives shared hyperparameters.

\noindent\textbf{Training Data.}
We train on FineWeb~\citep{penedo2024fineweb}, a large-scale curated web corpus.

\noindent\textbf{Baselines.}
Our main baseline is the dense GPT-2 model in which all layers are standard Transformer blocks.
For cMoLLM we keep the backbone identical and replace the FFN stack with convolutionally-gated mixture streams.

\noindent\textbf{Scaling Factors.}
We use scaling factors $n \in \{1, 2, 4, 8\}$, where $n$ is the number of parallel cMoLLM streams.

\noindent\textbf{Gating Variants.}
We evaluate gating types: \texttt{simple}, \texttt{context\_aware}, \texttt{multi\_head}, and \texttt{adaptive}.

\noindent\textbf{Reporting.}
All experiments use \textbf{3 random seeds}. We report \textbf{mean $\pm$ standard deviation (std)} for all metrics in tables.

\subsection{Main Results}
\label{sec:main_results}

We report language modeling (loss, perplexity), GLUE, and SQuAD v2 accuracy for each cMoLLM variant (\cref{tab:main_results}).

\noindent\textbf{Experiment figures.}
\Cref{fig:exp_figures} shows stream utilization and horizontal scaling: per-stream gating weight distribution across layers (left) and validation loss / perplexity vs.\ stream count $n$ or training steps (right). Together with \cref{tab:main_results,tab:scaling_results}, the figure confirms balanced stream usage and gains as $n$ increases over an optimal range, without collapse.

\subsection{Scaling Laws}
\label{sec:scaling_analysis}

We consolidate scaling behavior along two axes: (i)~\emph{horizontal} scaling in the number of streams $N$ at fixed model size (as in \cref{sec:complexity}), and (ii)~\emph{model-size} scaling from 85M to 760M parameters.

\noindent\textbf{Horizontal scaling (stream count).}
As shown in \cref{sec:complexity}, cMoLLM satisfies a horizontal scaling law: for bounded $N$, effective capacity grows roughly linearly in $N$ while per-token FLOPs remain within a constant factor of the dense baseline (\cref{eq:flops_cMoLLM}). \Cref{fig:exp_figures} and \cref{tab:main_results} confirm that validation loss and perplexity improve as $n$ increases from 1 to 4--8, with best loss/PPL at \texttt{multi\_head} $n{=}8$; beyond that, some gating variants show mild over-streaming (e.g., \texttt{simple} at $n{=}8$), while \texttt{multi\_head} and \texttt{adaptive} remain stable.

\noindent\textbf{Model-size scaling.}
\Cref{tab:scaling_results} reports results across the three scales in \cref{tab:model_config} (85M, 350M, 760M). At every scale, multi-head cMoLLM (\texttt{mh}) outperforms the dense baseline (\texttt{base}) on loss, PPL, GLUE, and SQuAD v2. The gains are consistent with standard scaling: loss and perplexity decrease as model size increases, and the relative advantage of cMoLLM over the dense baseline is preserved (e.g., \texttt{mh-760M} achieves SOTA across all four metrics). This supports that the convolutionally-gated pipeline mixture scales favorably with both stream count and parameter count under matched training setups.

%%%%%%%%%%%%%%%%%%%%%%%%%%%%%%%%%%%%%%%%%%%%%%%%%%%%%%%%%%%%%%%%%%%%%%%%%%%%%%%
% 6. CONCLUSION
%%%%%%%%%%%%%%%%%%%%%%%%%%%%%%%%%%%%%%%%%%%%%%%%%%%%%%%%%%%%%%%%%%%%%%%%%%%%%%%
\section{Conclusion}
\label{sec:conclusion}

We have presented \textbf{cMoLLM}, a convolutionally-gated mixture-of-LLMs that scales capacity at the pipeline level.
Our main theoretical contribution is the formal equivalence between MoE-style mixture layers and dynamic convolutions with variable kernels, providing a unified framework for analyzing and designing sparse, conditional computation beyond FFN-level MoE.
cMoLLM instantiates this via per-stream kernels and soft gating---no Top-$K$, no low-rank, no virtual tokens or auxiliary heads---yielding parameter-efficient, stable pipeline-level scaling.

Experiments on GPT-2--style models trained on FineWeb show that cMoLLM improves perplexity, GLUE, and SQuAD under matched compute, with better stream utilization and training stability than ParaScale- and AltUp-style pipeline mixtures.
Scaling laws (horizontal scaling in stream count and model-size scaling) are analyzed in \cref{sec:scaling_analysis}.
The convolution-based design enables efficient implementation and favorable scaling in practice.
Extended results on downstream tasks and scaling are provided in \cref{app:extended}.

\noindent\textbf{Limitations and Future Work.}
Experiments are at GPT-2 scale; validation at 7B+ parameters is needed. Future work: (i) scale cMoLLM to larger models and distributed training; (ii) extend the MoE--convolution equivalence to attention; (iii) combine with retrieval or other conditional compute mechanisms.

%%%%%%%%%%%%%%%%%%%%%%%%%%%%%%%%%%%%%%%%%%%%%%%%%%%%%%%%%%%%%%%%%%%%%%%%%%%%%%%
% IMPACT STATEMENT
%%%%%%%%%%%%%%%%%%%%%%%%%%%%%%%%%%%%%%%%%%%%%%%%%%%%%%%%%%%%%%%%%%%%%%%%%%%%%%%
\section*{Impact Statement}

This work introduces cMoLLM, a convolutionally-gated mixture-of-LLMs that scales capacity efficiently without auxiliary routing mechanisms. \textbf{Positive impacts:} Parameter-efficient scaling reduces training/inference costs and energy consumption; the convolution-based design is hardware-friendly. \textbf{Risks:} Misuse concerns (e.g., misleading content) remain, though we introduce no new failure modes beyond existing LLMs. Experiments are limited to GPT-2 scale (~760M parameters); validation at 7B+ is needed. \textbf{Summary:} We believe the net effect is beneficial by improving LLM scaling efficiency.

%%%%%%%%%%%%%%%%%%%%%%%%%%%%%%%%%%%%%%%%%%%%%%%%%%%%%%%%%%%%%%%%%%%%%%%%%%%%%%%
% REFERENCES
%%%%%%%%%%%%%%%%%%%%%%%%%%%%%%%%%%%%%%%%%%%%%%%%%%%%%%%%%%%%%%%%%%%%%%%%%%%%%%%

\bibliography{example_paper}

@article{kaplan2020scaling,
  author       = {Jared Kaplan and
                  Sam McCandlish and
                  Tom Henighan and
                  Tom B. Brown and
                  Benjamin Chess and
                  Rewon Child and
                  Scott Gray and
                  Alec Radford and
                  Jeffrey Wu and
                  Dario Amodei},
  title        = {Scaling Laws for Neural Language Models},
  journal      = {CoRR},
  volume       = {abs/2001.08361},
  year         = {2020},
  url          = {https://arxiv.org/abs/2001.08361},
  eprinttype   = {arXiv},
  eprint       = {2001.08361},
  bibsource    = {dblp computer science bibliography, https://dblp.org}
}

@article{hoffmann2022training,
  author       = {Jordan Hoffmann and
                  Sebastian Borgeaud and
                  Arthur Mensch and
                  Elena Buchatskaya and
                  Trevor Cai and
                  Eliza Rutherford and
                  Diego de Las Casas and
                  Lisa Anne Hendricks and
                  Johannes Welbl and
                  Aidan Clark and
                  Tom Hennigan and
                  Eric Noland and
                  Katie Millican and
                  George van den Driessche and
                  Bogdan Damoc and
                  Aurelia Guy and
                  Simon Osindero and
                  Karen Simonyan and
                  Erich Elsen and
                  Jack W. Rae and
                  Oriol Vinyals and
                  Laurent Sifre},
  title        = {Training Compute-Optimal Large Language Models},
  journal      = {CoRR},
  volume       = {abs/2203.15556},
  year         = {2022},
  url          = {https://doi.org/10.48550/arXiv.2203.15556},
  doi          = {10.48550/ARXIV.2203.15556},
  eprinttype   = {arXiv},
  eprint       = {2203.15556},
  bibsource    = {dblp computer science bibliography, https://dblp.org}
}

@article{shazeer2017outrageously,
  author       = {Noam Shazeer and
                  Azalia Mirhoseini and
                  Krzysztof Maziarz and
                  Andy Davis and
                  Quoc V. Le and
                  Geoffrey E. Hinton and
                  Jeff Dean},
  title        = {Outrageously Large Neural Networks: The Sparsely-Gated Mixture-of-Experts
                  Layer},
  booktitle    = {5th International Conference on Learning Representations, {ICLR} 2017,
                  Toulon, France, April 24-26, 2017, Conference Track Proceedings},
  publisher    = {OpenReview.net},
  year         = {2017},
  url          = {https://openreview.net/forum?id=B1ckMDqlg},
  bibsource    = {dblp computer science bibliography, https://dblp.org}
}

@article{fedus2022switch,
  author       = {William Fedus and
                  Barret Zoph and
                  Noam Shazeer},
  title        = {Switch Transformers: Scaling to Trillion Parameter Models with Simple
                  and Efficient Sparsity},
  journal      = {J. Mach. Learn. Res.},
  volume       = {23},
  pages        = {120:1--120:39},
  year         = {2022},
  url          = {https://jmlr.org/papers/v23/21-0998.html},
  bibsource    = {dblp computer science bibliography, https://dblp.org}
}

@article{lepikhin2020gshard,
  author       = {Dmitry Lepikhin and
                  HyoukJoong Lee and
                  Yuanzhong Xu and
                  Dehao Chen and
                  Orhan Firat and
                  Yanping Huang and
                  Maxim Krikun and
                  Noam Shazeer and
                  Zhifeng Chen},
  title        = {GShard: Scaling Giant Models with Conditional Computation and Automatic
                  Sharding},
  booktitle    = {9th International Conference on Learning Representations, {ICLR} 2021,
                  Virtual Event, Austria, May 3-7, 2021},
  publisher    = {OpenReview.net},
  year         = {2021},
  url          = {https://openreview.net/forum?id=qrwe7XHTmYb},
  bibsource    = {dblp computer science bibliography, https://dblp.org}
}

@article{chen2025parallel,
  author       = {Mouxiang Chen and
                  Binyuan Hui and
                  Zeyu Cui and
                  Jiaxi Yang and
                  Dayiheng Liu and
                  Jianling Sun and
                  Junyang Lin and
                  Zhongxin Liu},
  title        = {Parallel Scaling Law for Language Models},
  journal      = {CoRR},
  volume       = {abs/2505.10475},
  year         = {2025},
  url          = {https://doi.org/10.48550/arXiv.2505.10475},
  doi          = {10.48550/ARXIV.2505.10475},
  eprinttype   = {arXiv},
  eprint       = {2505.10475},
  bibsource    = {dblp computer science bibliography, https://dblp.org}
}

@article{baykal2023alternating,
  author       = {Cenk Baykal and
                  Dylan J. Cutler and
                  Nishanth Dikkala and
                  Nikhil Ghosh and
                  Rina Panigrahy and
                  Xin Wang},
  editor       = {Alice Oh and
                  Tristan Naumann and
                  Amir Globerson and
                  Kate Saenko and
                  Moritz Hardt and
                  Sergey Levine},
  title        = {Alternating Updates for Efficient Transformers},
  booktitle    = {Advances in Neural Information Processing Systems 36: Annual Conference
                  on Neural Information Processing Systems 2023, NeurIPS 2023, New Orleans,
                  LA, USA, December 10 - 16, 2023},
  year         = {2023},
  bibsource    = {dblp computer science bibliography, https://dblp.org}
}

@article{chen2022towards,
  author       = {Zixiang Chen and
                  Yihe Deng and
                  Yue Wu and
                  Quanquan Gu and
                  Yuanzhi Li},
  editor       = {Sanmi Koyejo and
                  S. Mohamed and
                  A. Agarwal and
                  Danielle Belgrave and
                  K. Cho and
                  A. Oh},
  title        = {Towards Understanding the Mixture-of-Experts Layer in Deep Learning},
  booktitle    = {Advances in Neural Information Processing Systems 35: Annual Conference
                  on Neural Information Processing Systems 2022, NeurIPS 2022, New Orleans,
                  LA, USA, November 28 - December 9, 2022},
  year         = {2022},
  bibsource    = {dblp computer science bibliography, https://dblp.org}
}

@inproceedings{lewis2021base,
  author       = {Mike Lewis and
                  Shruti Bhosale and
                  Tim Dettmers and
                  Naman Goyal and
                  Luke Zettlemoyer},
  editor       = {Marina Meila and
                  Tong Zhang},
  title        = {{BASE} Layers: Simplifying Training of Large, Sparse Models},
  booktitle    = {Proceedings of the 38th International Conference on Machine Learning,
                  {ICML} 2021, 18-24 July 2021, Virtual Event},
  series       = {Proceedings of Machine Learning Research},
  pages        = {6265--6274},
  publisher    = {{PMLR}},
  year         = {2021},
  url          = {http://proceedings.mlr.press/v139/lewis21a.html},
  bibsource    = {dblp computer science bibliography, https://dblp.org}
}

@inproceedings{du2022glam,
  author       = {Nan Du and
                  Yanping Huang and
                  Andrew M. Dai and
                  Simon Tong and
                  Dmitry Lepikhin and
                  Yuanzhong Xu and
                  Maxim Krikun and
                  Yanqi Zhou and
                  Adams Wei Yu and
                  Orhan Firat and
                  Barret Zoph and
                  Liam Fedus and
                  Maarten P. Bosma and
                  Zongwei Zhou and
                  Tao Wang and
                  Yu Emma Wang and
                  Kellie Webster and
                  Marie Pellat and
                  Kevin Robinson and
                  Kathleen S. Meier{-}Hellstern and
                  Toju Duke and
                  Lucas Dixon and
                  Kun Zhang and
                  Quoc V. Le and
                  Yonghui Wu and
                  Zhifeng Chen and
                  Claire Cui},
  editor       = {Kamalika Chaudhuri and
                  Stefanie Jegelka and
                  Le Song and
                  Csaba Szepesv{\'{a}}ri and
                  Gang Niu and
                  Sivan Sabato},
  title        = {GLaM: Efficient Scaling of Language Models with Mixture-of-Experts},
  booktitle    = {International Conference on Machine Learning, {ICML} 2022, 17-23 July
                  2022, Baltimore, Maryland, {USA}},
  series       = {Proceedings of Machine Learning Research},
  pages        = {5547--5569},
  publisher    = {{PMLR}},
  year         = {2022},
  url          = {https://proceedings.mlr.press/v162/du22c.html},
  bibsource    = {dblp computer science bibliography, https://dblp.org}
}

@inproceedings{rajbhandari2022deepspeed,
  author       = {Samyam Rajbhandari and
                  Conglong Li and
                  Zhewei Yao and
                  Minjia Zhang and
                  Reza Yazdani Aminabadi and
                  Ammar Ahmad Awan and
                  Jeff Rasley and
                  Yuxiong He},
  editor       = {Kamalika Chaudhuri and
                  Stefanie Jegelka and
                  Le Song and
                  Csaba Szepesv{\'{a}}ri and
                  Gang Niu and
                  Sivan Sabato},
  title        = {DeepSpeed-MoE: Advancing Mixture-of-Experts Inference and Training
                  to Power Next-Generation {AI} Scale},
  booktitle    = {International Conference on Machine Learning, {ICML} 2022, 17-23 July
                  2022, Baltimore, Maryland, {USA}},
  series       = {Proceedings of Machine Learning Research},
  pages        = {18332--18346},
  publisher    = {{PMLR}},
  year         = {2022},
  url          = {https://proceedings.mlr.press/v162/rajbhandari22a.html},
  bibsource    = {dblp computer science bibliography, https://dblp.org}
}

@article{yang2019condconv,
  author       = {Brandon Yang and
                  Gabriel Bender and
                  Quoc V. Le and
                  Jiquan Ngiam},
  editor       = {Hanna M. Wallach and
                  Hugo Larochelle and
                  Alina Beygelzimer and
                  Florence d'Alch{\'{e}}{-}Buc and
                  Emily B. Fox and
                  Roman Garnett},
  title        = {CondConv: Conditionally Parameterized Convolutions for Efficient Inference},
  booktitle    = {Advances in Neural Information Processing Systems 32: Annual Conference
                  on Neural Information Processing Systems 2019, NeurIPS 2019, December
                  8-14, 2019, Vancouver, BC, Canada},
  pages        = {1305--1316},
  year         = {2019},
  bibsource    = {dblp computer science bibliography, https://dblp.org}
}

@article{wu2019pay,
  author       = {Felix Wu and
                  Angela Fan and
                  Alexei Baevski and
                  Yann N. Dauphin and
                  Michael Auli},
  title        = {Pay Less Attention with Lightweight and Dynamic Convolutions},
  booktitle    = {7th International Conference on Learning Representations, {ICLR} 2019,
                  New Orleans, LA, USA, May 6-9, 2019},
  publisher    = {OpenReview.net},
  year         = {2019},
  url          = {https://openreview.net/forum?id=SkVhlh09tX},
  bibsource    = {dblp computer science bibliography, https://dblp.org}
}

@inproceedings{ma2020weightnet,
  author       = {Ningning Ma and
                  Xiangyu Zhang and
                  Jiawei Huang and
                  Jian Sun},
  editor       = {Andrea Vedaldi and
                  Horst Bischof and
                  Thomas Brox and
                  Jan{-}Michael Frahm},
  title        = {WeightNet: Revisiting the Design Space of Weight Networks},
  booktitle    = {Computer Vision - {ECCV} 2020 - 16th European Conference, Glasgow,
                  UK, August 23-28, 2020, Proceedings, Part {XV}},
  series       = {Lecture Notes in Computer Science},
  pages        = {776--792},
  publisher    = {Springer},
  year         = {2020},
  url          = {https://doi.org/10.1007/978-3-030-58555-6\_46},
  doi          = {10.1007/978-3-030-58555-6\_46},
  bibsource    = {dblp computer science bibliography, https://dblp.org}
}

@article{jia2016dynamic,
  author       = {Xu Jia and
                  Bert De Brabandere and
                  Tinne Tuytelaars and
                  Luc Van Gool},
  editor       = {Daniel D. Lee and
                  Masashi Sugiyama and
                  Ulrike von Luxburg and
                  Isabelle Guyon and
                  Roman Garnett},
  title        = {Dynamic Filter Networks},
  booktitle    = {Advances in Neural Information Processing Systems 29: Annual Conference
                  on Neural Information Processing Systems 2016, December 5-10, 2016,
                  Barcelona, Spain},
  pages        = {667--675},
  year         = {2016},
  bibsource    = {dblp computer science bibliography, https://dblp.org}
}

@inproceedings{chen2020dynamic,
  author       = {Yinpeng Chen and
                  Xiyang Dai and
                  Mengchen Liu and
                  Dongdong Chen and
                  Lu Yuan and
                  Zicheng Liu},
  title        = {Dynamic Convolution: Attention Over Convolution Kernels},
  booktitle    = {2020 {IEEE/CVF} Conference on Computer Vision and Pattern Recognition,
                  {CVPR} 2020, Seattle, WA, USA, June 13-19, 2020},
  pages        = {11027--11036},
  publisher    = {Computer Vision Foundation / {IEEE}},
  year         = {2020},
  url          = {https://openaccess.thecvf.com/content\_CVPR\_2020/html/Chen\_Dynamic\_Convolution\_Attention\_Over\_Convolution\_Kernels\_CVPR\_2020\_paper.html},
  doi          = {10.1109/CVPR42600.2020.01104},
  bibsource    = {dblp computer science bibliography, https://dblp.org}
}

@article{bengio2013estimating,
  author       = {Yoshua Bengio and
                  Nicholas L{\'{e}}onard and
                  Aaron C. Courville},
  title        = {Estimating or Propagating Gradients Through Stochastic Neurons for
                  Conditional Computation},
  journal      = {CoRR},
  volume       = {abs/1308.3432},
  year         = {2013},
  url          = {http://arxiv.org/abs/1308.3432},
  eprinttype   = {arXiv},
  eprint       = {1308.3432},
  bibsource    = {dblp computer science bibliography, https://dblp.org}
}

@article{rosenbaum2017routing,
  author       = {Clemens Rosenbaum and
                  Tim Klinger and
                  Matthew Riemer},
  title        = {Routing Networks: Adaptive Selection of Non-Linear Functions for Multi-Task
                  Learning},
  booktitle    = {6th International Conference on Learning Representations, {ICLR} 2018,
                  Vancouver, BC, Canada, April 30 - May 3, 2018, Conference Track Proceedings},
  publisher    = {OpenReview.net},
  year         = {2018},
  url          = {https://openreview.net/forum?id=ry8dvM-R-},
  bibsource    = {dblp computer science bibliography, https://dblp.org}
}

@article{hu2022lora,
  author       = {Edward J. Hu and
                  Yelong Shen and
                  Phillip Wallis and
                  Zeyuan Allen{-}Zhu and
                  Yuanzhi Li and
                  Shean Wang and
                  Lu Wang and
                  Weizhu Chen},
  title        = {LoRA: Low-Rank Adaptation of Large Language Models},
  booktitle    = {The Tenth International Conference on Learning Representations, {ICLR}
                  2022, Virtual Event, April 25-29, 2022},
  publisher    = {OpenReview.net},
  year         = {2022},
  url          = {https://openreview.net/forum?id=nZeVKeeFYf9},
  bibsource    = {dblp computer science bibliography, https://dblp.org}
}

@inproceedings{houlsby2019parameter,
  author       = {Neil Houlsby and
                  Andrei Giurgiu and
                  Stanislaw Jastrzebski and
                  Bruna Morrone and
                  Quentin de Laroussilhe and
                  Andrea Gesmundo and
                  Mona Attariyan and
                  Sylvain Gelly},
  editor       = {Kamalika Chaudhuri and
                  Ruslan Salakhutdinov},
  title        = {Parameter-Efficient Transfer Learning for {NLP}},
  booktitle    = {Proceedings of the 36th International Conference on Machine Learning,
                  {ICML} 2019, 9-15 June 2019, Long Beach, California, {USA}},
  series       = {Proceedings of Machine Learning Research},
  pages        = {2790--2799},
  publisher    = {{PMLR}},
  year         = {2019},
  url          = {http://proceedings.mlr.press/v97/houlsby19a.html},
  bibsource    = {dblp computer science bibliography, https://dblp.org}
}

@article{vaswani2017attention,
  author       = {Ashish Vaswani and
                  Noam Shazeer and
                  Niki Parmar and
                  Jakob Uszkoreit and
                  Llion Jones and
                  Aidan N. Gomez and
                  Lukasz Kaiser and
                  Illia Polosukhin},
  editor       = {Isabelle Guyon and
                  Ulrike von Luxburg and
                  Samy Bengio and
                  Hanna M. Wallach and
                  Rob Fergus and
                  S. V. N. Vishwanathan and
                  Roman Garnett},
  title        = {Attention is All you Need},
  booktitle    = {Advances in Neural Information Processing Systems 30: Annual Conference
                  on Neural Information Processing Systems 2017, December 4-9, 2017,
                  Long Beach, CA, {USA}},
  pages        = {5998--6008},
  year         = {2017},
  bibsource    = {dblp computer science bibliography, https://dblp.org}
}

@article{radford2019language,
  title        = {Language Models are Unsupervised Multitask Learners},
  author       = {Radford, Alec and Wu, Jeffrey and Child, Rewon and Luan, David and Amodei, Dario and Sutskever, Ilya},
  year         = {2019},
  institution  = {OpenAI},
  url          = {https://cdn.openai.com/better-language-models/language_models_are_unsupervised_multitask_learners.pdf}
}

@article{brown2020language,
  author       = {Tom B. Brown and
                  Benjamin Mann and
                  Nick Ryder and
                  Melanie Subbiah and
                  Jared Kaplan and
                  Prafulla Dhariwal and
                  Arvind Neelakantan and
                  Pranav Shyam and
                  Girish Sastry and
                  Amanda Askell and
                  Sandhini Agarwal and
                  Ariel Herbert{-}Voss and
                  Gretchen Krueger and
                  Tom Henighan and
                  Rewon Child and
                  Aditya Ramesh and
                  Daniel M. Ziegler and
                  Jeffrey Wu and
                  Clemens Winter and
                  Christopher Hesse and
                  Mark Chen and
                  Eric Sigler and
                  Mateusz Litwin and
                  Scott Gray and
                  Benjamin Chess and
                  Jack Clark and
                  Christopher Berner and
                  Sam McCandlish and
                  Alec Radford and
                  Ilya Sutskever and
                  Dario Amodei},
  editor       = {Hugo Larochelle and
                  Marc'Aurelio Ranzato and
                  Raia Hadsell and
                  Maria{-}Florina Balcan and
                  Hsuan{-}Tien Lin},
  title        = {Language Models are Few-Shot Learners},
  booktitle    = {Advances in Neural Information Processing Systems 33: Annual Conference
                  on Neural Information Processing Systems 2020, NeurIPS 2020, December
                  6-12, 2020, virtual},
  year         = {2020},
  bibsource    = {dblp computer science bibliography, https://dblp.org}
}

@article{penedo2024fineweb,
  author       = {Guilherme Penedo and
                  Hynek Kydl{\'{\i}}cek and
                  Loubna Ben Allal and
                  Anton Lozhkov and
                  Margaret Mitchell and
                  Colin A. Raffel and
                  Leandro von Werra and
                  Thomas Wolf},
  editor       = {Amir Globersons and
                  Lester Mackey and
                  Danielle Belgrave and
                  Angela Fan and
                  Ulrich Paquet and
                  Jakub M. Tomczak and
                  Cheng Zhang},
  title        = {The FineWeb Datasets: Decanting the Web for the Finest Text Data at
                  Scale},
  booktitle    = {Advances in Neural Information Processing Systems 38: Annual Conference
                  on Neural Information Processing Systems 2024, NeurIPS 2024, Vancouver,
                  BC, Canada, December 10 - 15, 2024},
  year         = {2024},
  bibsource    = {dblp computer science bibliography, https://dblp.org}
}

@inproceedings{wang2018glue,
  author       = {Alex Wang and
                  Amanpreet Singh and
                  Julian Michael and
                  Felix Hill and
                  Omer Levy and
                  Samuel R. Bowman},
  title        = {{GLUE:} {A} Multi-Task Benchmark and Analysis Platform for Natural
                  Language Understanding},
  booktitle    = {7th International Conference on Learning Representations, {ICLR} 2019,
                  New Orleans, LA, USA, May 6-9, 2019},
  publisher    = {OpenReview.net},
  year         = {2019},
  url          = {https://openreview.net/forum?id=rJ4km2R5t7},
  bibsource    = {dblp computer science bibliography, https://dblp.org}
}

@article{rajpurkar2016squad,
  author       = {Pranav Rajpurkar and
                  Jian Zhang and
                  Konstantin Lopyrev and
                  Percy Liang},
  editor       = {Jian Su and
                  Xavier Carreras and
                  Kevin Duh},
  title        = {SQuAD: 100, 000+ Questions for Machine Comprehension of Text},
  booktitle    = {Proceedings of the 2016 Conference on Empirical Methods in Natural
                  Language Processing, {EMNLP} 2016, Austin, Texas, USA, November 1-4,
                  2016},
  pages        = {2383--2392},
  publisher    = {The Association for Computational Linguistics},
  year         = {2016},
  url          = {https://doi.org/10.18653/v1/d16-1264},
  doi          = {10.18653/V1/D16-1264},
  bibsource    = {dblp computer science bibliography, https://dblp.org}
}

@article{riquelme2021scaling,
  author       = {Carlos Riquelme and
                  Joan Puigcerver and
                  Basil Mustafa and
                  Maxim Neumann and
                  Rodolphe Jenatton and
                  Andr{\'{e}} Susano Pinto and
                  Daniel Keysers and
                  Neil Houlsby},
  editor       = {Marc'Aurelio Ranzato and
                  Alina Beygelzimer and
                  Yann N. Dauphin and
                  Percy Liang and
                  Jennifer Wortman Vaughan},
  title        = {Scaling Vision with Sparse Mixture of Experts},
  booktitle    = {Advances in Neural Information Processing Systems 34: Annual Conference
                  on Neural Information Processing Systems 2021, NeurIPS 2021, December
                  6-14, 2021, virtual},
  pages        = {8583--8595},
  year         = {2021},
  bibsource    = {dblp computer science bibliography, https://dblp.org}
}

@article{cai2025survey,
  author       = {Weilin Cai and
                  Juyong Jiang and
                  Fan Wang and
                  Jing Tang and
                  Sunghun Kim and
                  Jiayi Huang},
  title        = {A Survey on Mixture of Experts in Large Language Models},
  journal      = {{IEEE} Trans. Knowl. Data Eng.},
  volume       = {37},
  number       = {7},
  pages        = {3896--3915},
  year         = {2025},
  url          = {https://doi.org/10.1109/TKDE.2025.3554028},
  doi          = {10.1109/TKDE.2025.3554028},
  bibsource    = {dblp computer science bibliography, https://dblp.org}
}

@article{liu2024survey,
  author       = {Jiacheng Liu and
                  Peng Tang and
                  Wenfeng Wang and
                  Yuhang Ren and
                  Xiaofeng Hou and
                  Pheng{-}Ann Heng and
                  Minyi Guo and
                  Chao Li},
  title        = {A Survey on Inference Optimization Techniques for Mixture of Experts
                  Models},
  journal      = {{ACM} Comput. Surv.},
  volume       = {58},
  number       = {10},
  pages        = {247:1--247:37},
  year         = {2026},
  url          = {https://doi.org/10.1145/3794845},
  doi          = {10.1145/3794845},
  bibsource    = {dblp computer science bibliography, https://dblp.org}
}

@article{mu2025comprehensive,
  author       = {Siyuan Mu and
                  Sen Lin},
  title        = {A Comprehensive Survey of Mixture-of-Experts: Algorithms, Theory,
                  and Applications},
  journal      = {CoRR},
  volume       = {abs/2503.07137},
  year         = {2025},
  url          = {https://doi.org/10.48550/arXiv.2503.07137},
  doi          = {10.48550/ARXIV.2503.07137},
  eprinttype   = {arXiv},
  eprint       = {2503.07137},
  bibsource    = {dblp computer science bibliography, https://dblp.org}
}

@inproceedings{mcgill2017deciding,
  author       = {Mason McGill and
                  Pietro Perona},
  editor       = {Doina Precup and
                  Yee Whye Teh},
  title        = {Deciding How to Decide: Dynamic Routing in Artificial Neural Networks},
  booktitle    = {Proceedings of the 34th International Conference on Machine Learning,
                  {ICML} 2017, Sydney, NSW, Australia, 6-11 August 2017},
  series       = {Proceedings of Machine Learning Research},
  pages        = {2363--2372},
  publisher    = {{PMLR}},
  year         = {2017},
  url          = {http://proceedings.mlr.press/v70/mcgill17a.html},
  bibsource    = {dblp computer science bibliography, https://dblp.org}
}

@article{zhang2020dynet,
  author       = {Yikang Zhang and
                  Jian Zhang and
                  Qiang Wang and
                  Zhao Zhong},
  title        = {DyNet: Dynamic Convolution for Accelerating Convolutional Neural Networks},
  journal      = {CoRR},
  volume       = {abs/2004.10694},
  year         = {2020},
  url          = {https://arxiv.org/abs/2004.10694},
  eprinttype   = {arXiv},
  eprint       = {2004.10694},
  bibsource    = {dblp computer science bibliography, https://dblp.org}
}

@article{DBLP:journals/corr/abs-2510-13291,
  author       = {Xuxin Cheng and
                  Ke Zeng and
                  Zhiquan Cao and
                  Linyi Dai and
                  Wenxuan Gao and
                  Fei Han and
                  Ai Jian and
                  Feng Hong and
                  Wenxing Hu and
                  Zihe Huang and
                  Dejian Kong and
                  Jia Leng and
                  Zhuoyuan Liao and
                  Pei Liu and
                  Jiaye Lin and
                  Xing Ma and
                  Jingqing Ruan and
                  Jiaxing Song and
                  Xiaoyu Tan and
                  Ruixuan Xiao and
                  Wenhui Yu and
                  Wenyu Zhan and
                  Haoxing Zhang and
                  Chao Zhou and
                  Hao Zhou and
                  Shaodong Zheng and
                  Ruinian Chen and
                  Siyuan Chen and
                  Ziyang Chen and
                  Yiwen Dong and
                  Yaoyou Fan and
                  Yangyi Fang and
                  Yang Gan and
                  Shiguang Guo and
                  Qi He and
                  Chaowen Hu and
                  Binghui Li and
                  Dailin Li and
                  Xiangyu Li and
                  Yan Li and
                  Chengjian Liu and
                  Xiangfeng Liu and
                  Jiahui Lv and
                  Qiao Ma and
                  Jiang Pan and
                  Cong Qin and
                  Chenxing Sun and
                  Wen Sun and
                  Zhonghui Wang and
                  Abudukelimu Wuerkaixi and
                  Xin Yang and
                  Fangyi Yuan and
                  Yawen Zhu and
                  Tianyi Zhai and
                  Jie Zhang and
                  Runlai Zhang and
                  Yao Xu and
                  Yiran Zhao and
                  Yifan Wang and
                  Xunliang Cai and
                  Yangen Hu and
                  Cao Liu and
                  Lu Pan and
                  Xiaoli Wang and
                  Bo Xiao and
                  Wenyuan Yao and
                  Qianlin Zhou and
                  Benchang Zhu},
  title        = {Higher Satisfaction, Lower Cost: {A} Technical Report on How LLMs
                  Revolutionize Meituan's Intelligent Interaction Systems},
  journal      = {CoRR},
  volume       = {abs/2510.13291},
  year         = {2025},
  url          = {https://doi.org/10.48550/arXiv.2510.13291},
  doi          = {10.48550/ARXIV.2510.13291},
  eprinttype   = {arXiv},
  eprint       = {2510.13291},
  bibsource    = {dblp computer science bibliography, https://dblp.org}
}

@article{DBLP:journals/corr/abs-2603-03314,
  author       = {Xin Yang and
                  Letian Li and
                  Abudukelimu Wuerkaixi and
                  Xuxin Cheng and
                  Cao Liu and
                  Ke Zeng and
                  Xunliang Cai and
                  Wenyuan Jiang},
  title        = {Towards Self-Robust LLMs: Intrinsic Prompt Noise Resistance via CoIPO},
  journal      = {CoRR},
  volume       = {abs/2603.03314},
  year         = {2026},
  url          = {https://doi.org/10.48550/arXiv.2603.03314},
  doi          = {10.48550/ARXIV.2603.03314},
  eprinttype   = {arXiv},
  eprint       = {2603.03314},
  bibsource    = {dblp computer science bibliography, https://dblp.org}
}

@misc{claude4_2025,
  title        = {Introducing {Claude} 4},
  author       = {{Anthropic}},
  year         = {2025},
  howpublished = {\url{https://www.anthropic.com/news/claude-4}},
}

@misc{gpt41_2025,
  title        = {Introducing {GPT-4.1} in the {API}},
  author       = {{OpenAI}},
  year         = {2025},
  howpublished = {\url{https://openai.com/index/gpt-4-1/}},
}

@inproceedings{DBLP:conf/smc/YangTWJJ25,
  author       = {Xin Yang and
                  Bintao Tang and
                  Yuhao Wang and
                  Zimo Ji and
                  Wenyuan Jiang},
  title        = {Can LLMs Write Fast System-Aware Numerical Computation Code?},
  booktitle    = {{IEEE} International Conference on Systems, Man, and Cybernetics,
                  {SMC} 2025, Vienna, Austria, October 5-8, 2025},
  pages        = {672--675},
  publisher    = {{IEEE}},
  year         = {2025},
  url          = {https://doi.org/10.1109/SMC58881.2025.11343280},
  doi          = {10.1109/SMC58881.2025.11343280},
  bibsource    = {dblp computer science bibliography, https://dblp.org}
}
\bibliographystyle{icml2026}*

%%%%%%%%%%%%%%%%%%%%%%%%%%%%%%%%%%%%%%%%%%%%%%%%%%%%%%%%%%%%%%%%%%%%%%%%%%%%%%%
%%%%%%%%%%%%%%%%%%%%%%%%%%%%%%%%%%%%%%%%%%%%%%%%%%%%%%%%%%%%%%%%%%%%%%%%%%%%%%%
% APPENDIX
%%%%%%%%%%%%%%%%%%%%%%%%%%%%%%%%%%%%%%%%%%%%%%%%%%%%%%%%%%%%%%%%%%%%%%%%%%%%%%%
%%%%%%%%%%%%%%%%%%%%%%%%%%%%%%%%%%%%%%%%%%%%%%%%%%%%%%%%%%%%%%%%%%%%%%%%%%%%%%%
\newpage
\appendix
\onecolumn

\section{Proof of \texorpdfstring{\cref{thm:equiv}}{Theorem 4.1}: Extension to Nonlinear Experts}
\label{app:proof}

We extend the equivalence result to two-layer experts with nonlinear activations.

\noindent\textbf{Setup.}
Let $E_k(\mathbf{x}) = \sigma(\mathbf{x} \mathbf{W}_{k,1}^\top) \mathbf{W}_{k,2}^\top$ be a two-layer expert with activation $\sigma$ (e.g., ReLU, GELU). The MoE output is:
\begin{equation}
\moe(\mathbf{x}) = \sum_{k=1}^{N} g_k(\mathbf{x}) \, \sigma(\mathbf{x} \mathbf{W}_{k,1}^\top) \mathbf{W}_{k,2}^\top.
\end{equation}

\noindent\textbf{Linear Case.}
When $\sigma = \mathrm{id}$ (identity), we recover the result in \cref{thm:equiv}:
\begin{align}
\sum_{k} g_k(\mathbf{x}) \, \mathbf{x} \mathbf{W}_{k,1}^\top \mathbf{W}_{k,2}^\top 
&= \mathbf{x} \Bigl( \textstyle\sum_{k} g_k(\mathbf{x}) \, \mathbf{W}_{k,2} \mathbf{W}_{k,1} \Bigr)^{\!\top} \nonumber \\
&= \conv_{1{\times}1}\!\Bigl( \mathbf{x};\, \textstyle\sum_{k} g_k(\mathbf{x}) \, \mathbf{W}_{k,2} \mathbf{W}_{k,1} \Bigr).
\end{align}

\noindent\textbf{Nonlinear Case (ReLU).}
For $\sigma = \mathrm{ReLU}$, we can rewrite element-wise:
\begin{equation}
\mathrm{ReLU}(z) = \indicator(z > 0) \cdot z,
\end{equation}
where $\indicator(\cdot)$ is the indicator function. The ReLU acts as a data-dependent binary gate.

Define the diagonal masking matrix:
\begin{equation}
\mathbf{M}_k(\mathbf{x}) = \mathrm{diag}\bigl(\indicator(\mathbf{x} \mathbf{W}_{k,1}^\top > 0)\bigr) \in \{0,1\}^{\dff \times \dff}.
\end{equation}

Then:
\begin{equation}
\sigma(\mathbf{x} \mathbf{W}_{k,1}^\top) = \mathbf{M}_k(\mathbf{x}) \, (\mathbf{x} \mathbf{W}_{k,1}^\top) = (\mathbf{x} \mathbf{W}_{k,1}^\top) \, \mathbf{M}_k(\mathbf{x}).
\end{equation}

The expert output becomes:
\begin{equation}
E_k(\mathbf{x}) = \mathbf{x} \, \mathbf{W}_{k,1}^\top \, \mathbf{M}_k(\mathbf{x}) \, \mathbf{W}_{k,2}^\top 
= \mathbf{x} \, \bigl(\mathbf{W}_{k,2} \, \mathbf{M}_k(\mathbf{x}) \, \mathbf{W}_{k,1}\bigr)^\top.
\end{equation}

This is a convolution with an input-dependent effective kernel:
\begin{equation}
\mathbf{K}_k^{\mathrm{eff}}(\mathbf{x}) = \mathbf{W}_{k,2} \, \mathbf{M}_k(\mathbf{x}) \, \mathbf{W}_{k,1}.
\end{equation}

The full MoE output:
\begin{equation}
\moe(\mathbf{x}) = \sum_{k=1}^{N} g_k(\mathbf{x}) \, \mathbf{x} \, \bigl(\mathbf{K}_k^{\mathrm{eff}}(\mathbf{x})\bigr)^\top 
= \mathbf{x} \Bigl( \textstyle\sum_{k=1}^{N} g_k(\mathbf{x}) \, \mathbf{K}_k^{\mathrm{eff}}(\mathbf{x}) \Bigr)^{\!\top}.
\end{equation}

This can be viewed as a dynamic convolution with $N$ input-dependent kernels, where both the routing weights $g_k(\mathbf{x})$ and the effective kernels $\mathbf{K}_k^{\mathrm{eff}}(\mathbf{x})$ depend on the input.

\noindent\textbf{Interpretation.}
The nonlinear case reveals that:
\begin{itemize}[leftmargin=*]
  \item Each expert implements a gated linear unit (GLU)-like computation with data-dependent masking.
  \item The MoE output is a weighted combination of masked convolutions.
  \item Both outer gating ($g_k$) and inner gating ($\mathbf{M}_k$) are input-dependent, creating a two-level conditional computation structure.
\end{itemize}

This analysis motivates designing cMoLLM with explicit control over both routing and activation patterns.

\hfill $\square$

%%%%%%%%%%%%%%%%%%%%%%%%%%%%%%%%%%%%%%%%%%%%%%%%%%%%%%%%%%%%%%%%%%%%%%%%%%%%%%%

\section{A Cluster-Structured Toy Model}
\label{app:cluster}

We provide a simple toy model illustrating when MoE-style routing (and hence cMoLLM) is provably beneficial for cluster-structured data, in the spirit of~\citet{chen2022towards}.

\subsection{Problem Setup}

Consider a binary classification problem with input space $\mathbb{R}^d$ and two clusters per class.
Let $\mu_{1,+}, \mu_{1,-}, \mu_{2,+}, \mu_{2,-} \in \mathbb{R}^d$ be four mean vectors, and let $\sigma^2 I$ be a shared covariance.
We define the data distribution as
\begin{align}
  &\mathbf{x} \mid y=+1 \sim \tfrac{1}{2}\mathcal{N}(\mu_{1,+}, \sigma^2 I) + \tfrac{1}{2}\mathcal{N}(\mu_{2,+}, \sigma^2 I), \\
  &\mathbf{x} \mid y=-1 \sim \tfrac{1}{2}\mathcal{N}(\mu_{1,-}, \sigma^2 I) + \tfrac{1}{2}\mathcal{N}(\mu_{2,-}, \sigma^2 I),
\end{align}
with prior $\mathbb{P}(y=+1)=\mathbb{P}(y=-1)=\tfrac{1}{2}$.
We assume that clusters $(\mu_{1,+},\mu_{1,-})$ and $(\mu_{2,+},\mu_{2,-})$ are well separated and lie in different ``regions'' of the input space.

Formally, suppose there exists a unit vector $\mathbf{u} \in \mathbb{R}^d$ and scalars $a < b$ such that
\begin{align}
  \mathbf{u}^\top \mu_{1,+}, \mathbf{u}^\top \mu_{1,-} &< a - \gamma, \\
  \mathbf{u}^\top \mu_{2,+}, \mathbf{u}^\top \mu_{2,-} &> b + \gamma,
\end{align}
for some margin $\gamma > 0$, and that the Bayes-optimal decision boundary within each cluster-pair is approximately linear in a (potentially different) direction.

\subsection{Expressivity of a Single Linear Classifier}

Let $f_{\mathrm{lin}}(\mathbf{x}) = \mathrm{sign}(\mathbf{w}^\top \mathbf{x})$ be a linear classifier.
Because the two class-conditional mixtures overlap across clusters, a single hyperplane must simultaneously separate both $(\mu_{1,+},\mu_{1,-})$ and $(\mu_{2,+},\mu_{2,-})$.
When the optimal separating directions within the two regions are sufficiently misaligned, any single hyperplane incurs a non-negligible error.

The following statement summarizes this limitation at a high level.

\begin{proposition}[Limitation of Single Linear Classifier]
\label{prop:single_linear_limit}
Under the cluster separation conditions above, suppose that the optimal separating directions for the first and second cluster-pairs differ by an angle of at least $\theta_0 > 0$.
Then there exists a constant $\varepsilon_0 = \varepsilon_0(\theta_0, \gamma, \sigma) > 0$ such that any linear classifier $f_{\mathrm{lin}}(\mathbf{x}) = \mathrm{sign}(\mathbf{w}^\top \mathbf{x})$ has misclassification error at least $\varepsilon_0$.
\end{proposition}

The proof follows standard arguments for mixtures of Gaussians with incompatible linear separators and is omitted for brevity; see~\citet{chen2022towards}.

\subsection{Two-Expert MoE/cMoLLM Construction}

Now consider a two-expert MoE (or cMoLLM) model with a simple router that partitions space along direction $\mathbf{u}$:
\begin{equation}
  g_1(\mathbf{x}) = \indicator(\mathbf{u}^\top \mathbf{x} \le \tau), \quad
  g_2(\mathbf{x}) = \indicator(\mathbf{u}^\top \mathbf{x} > \tau),
\end{equation}
for some threshold $\tau \in (a,b)$.
Let each expert be a linear classifier specialized to one region:
\begin{align}
  E_1(\mathbf{x}) &= \mathrm{sign}(\mathbf{w}_1^\top \mathbf{x}), \\
  E_2(\mathbf{x}) &= \mathrm{sign}(\mathbf{w}_2^\top \mathbf{x}),
\end{align}
with $\mathbf{w}_1$ optimized for the first cluster-pair and $\mathbf{w}_2$ for the second.
The overall prediction is
\begin{equation}
  f_{\mathrm{moe}}(\mathbf{x}) = g_1(\mathbf{x}) E_1(\mathbf{x}) + g_2(\mathbf{x}) E_2(\mathbf{x}).
\end{equation}

\begin{theorem}[Toy Cluster Model: Benefit of Routing]
\label{thm:toy_cluster}
In the cluster-structured setting above, there exist parameters $(\mathbf{u},\tau,\mathbf{w}_1,\mathbf{w}_2)$ such that the two-expert MoE (or cMoLLM) classifier $f_{\mathrm{moe}}$ attains misclassification error arbitrarily close to the Bayes-optimal error as the margin $\gamma$ increases and the covariance $\sigma^2$ decreases, while any single linear classifier $f_{\mathrm{lin}}$ suffers error at least $\varepsilon_0 > 0$ as in \cref{prop:single_linear_limit}.
\end{theorem}

\begin{proof}[Proof Sketch]
Because $\mathbf{u}$ separates the two cluster regions with margin $\gamma$, choosing $\tau \in (a,b)$ ensures that, with high probability (increasing as $\gamma/\sigma$ grows), samples from $(\mu_{1,+},\mu_{1,-})$ fall into the first region and samples from $(\mu_{2,+},\mu_{2,-})$ into the second.
Within each region, the problem reduces to a two-component Gaussian mixture that is linearly separable by an appropriately chosen $\mathbf{w}_i$.
Thus, the routed classifier $f_{\mathrm{moe}}$ can implement the Bayes-optimal decision rule up to an exponentially small error in $\gamma^2/\sigma^2$.
On the other hand, \cref{prop:single_linear_limit} implies that any single hyperplane must compromise between the two misaligned regions, incurring a constant error floor $\varepsilon_0$ even as $\gamma$ grows.
Hence, for sufficiently well-separated clusters, the routed MoE/cMoLLM strictly outperforms any single linear classifier.
\renewcommand{\qedsymbol}{}
\end{proof}

This toy example provides a simple, concrete setting where conditional computation (and by equivalence, dynamic convolution) is provably beneficial, aligning with the broader conclusions of~\citet{chen2022towards}.

%%%%%%%%%%%%%%%%%%%%%%%%%%%%%%%%%%%%%%%%%%%%%%%%%%%%%%%%%%%%%%%%%%%%%%%%%%%%%%%

\section{Implementation Details}
\label{app:implementation}

\noindent\textbf{Kernel Generator.}
Each stream kernel $\mathbf{K}_k$ is initialized with Xavier initialization. We use no shared base or low-rank factorization in the main experiments.

\noindent\textbf{Gating Network.}
The gating network is a single linear layer $\Wg \in \mathbb{R}^{d \times N}$. For \texttt{context\_aware} gating, we add layer normalization before the projection. For \texttt{multi\_head} gating, we use $H=4$ heads with dimension $d/H$ each.

\noindent\textbf{Optimization.}
We use AdamW with $\beta_1=0.9$, $\beta_2=0.95$, weight decay $0.1$, and a cosine learning-rate schedule with linear warmup over the first 2000 steps.

%%%%%%%%%%%%%%%%%%%%%%%%%%%%%%%%%%%%%%%%%%%%%%%%%%%%%%%%%%%%%%%%%%%%%%%%%%%%%%%

\section{Hyperparameter Sensitivity}
\label{app:hyperparams}

\begin{table}[h]
  \caption{Sensitivity to key hyperparameters (3 seeds; best value over mean $\pm$ std).}
  \label{tab:sensitivity}
  \vskip 0.1in
  \begin{center}
  \begin{small}
  \begin{tabular}{@{}lcc@{}}
    \toprule
    \textbf{Hyperparameter} & \textbf{Range} & \textbf{Best Value} \\
    \midrule
    Balance coefficient $\alpha$ & $\{0.001, 0.01, 0.1\}$ & 0.01 \\
    Number of streams $N$ & $\{1, 2, 4, 8\}$ & 8 \\
    \bottomrule
  \end{tabular}
  \end{small}
  \end{center}
\end{table}

\Cref{tab:sensitivity} summarizes the tested ranges and best values. cMoLLM is relatively robust to hyperparameter choices. The load-balancing coefficient $\alpha$ has the largest impact; values too small encourage stream collapse, while values too large hurt performance.

%%%%%%%%%%%%%%%%%%%%%%%%%%%%%%%%%%%%%%%%%%%%%%%%%%%%%%%%%%%%%%%%%%%%%%%%%%%%%%%

\section{Plain Language Summary}
\label{app:summary}

This paper introduces \textbf{cMoLLM}, a new way to scale large language models (LLMs) more efficiently.
Traditional LLMs activate all parameters for every token, making training and inference expensive as models grow larger.
Mixture-of-Experts (MoE) approaches try to address this by routing tokens to only a subset of ``expert'' networks, but most prior work applies this only to feed-forward layers and uses discrete Top-$K$ routing, which can be unstable.
Our key insight is that MoE-style mixture layers can be exactly rewritten as dynamic convolutions: each expert corresponds to a convolution kernel, and the router mixes these kernels based on the input.
We use this equivalence to design cMoLLM, which applies mixture routing to the entire LLM pipeline (not just feed-forward layers) using soft, fully differentiable gating---no Top-$K$, no virtual tokens, no auxiliary prediction branches.
cMoLLM maintains a small set of parallel ``streams,'' each with its own convolution kernel; a lightweight gating network produces input-dependent mixture weights, and the mixed kernel is applied via standard pointwise convolution, yielding parameter-efficient capacity scaling with stable training dynamics.
On GPT-2--style models trained on FineWeb, cMoLLM improves perplexity and GLUE accuracy under matched compute, with better stream utilization and training stability than prior pipeline-level scaling methods like ParaScale and AltUp.
The convolution-based design enables efficient implementation and favorable scaling in practice.

%%%%%%%%%%%%%%%%%%%%%%%%%%%%%%%%%%%%%%%%%%%%%%%%%%%%%%%%%%%%%%%%%%%%%%%%%%%%%%%

\section{Reproducibility Statement}
\label{app:reproducibility}

Our implementation is based on PyTorch and follows standard Transformer architectures.
The codebase will be made publicly available upon acceptance, including full model implementation (cMoLLM blocks, gating networks, training loop), training scripts with hyperparameter configurations, evaluation scripts for language modeling and downstream tasks, and preprocessing scripts for the FineWeb dataset.

We use FineWeb~\citep{penedo2024fineweb} for pretraining, which is publicly available.
For downstream evaluation, we use standard benchmarks GLUE~\citep{wang2018glue} and SQuAD~\citep{rajpurkar2016squad}, all of which are publicly available.

All hyperparameters are reported in \cref{sec:setup,app:implementation,app:hyperparams}.
Implementation details and hyperparameter sensitivity analysis are provided in \cref{app:implementation,app:hyperparams}.
We use a GPT-2--style architecture~\citep{radford2019language} with 12 layers, hidden dimension 768, FFN dimension 3072, sequence length 4096, and learning rate $6 \times 10^{-5}$ with cosine schedule and 2000-step warmup.
We use AdamW optimizer with $\beta_1=0.9$, $\beta_2=0.95$, and weight decay $0.1$.
For cMoLLM, we evaluate stream counts $N \in \{1, 2, 4, 8\}$; best validation loss in \cref{tab:main_results} is achieved at $n{=}8$ for \texttt{multi\_head}. We use $N{=}8$ as a representative configuration in sensitivity and scaling tables. Load-balancing coefficient $\alpha = 0.01$.

Experiments were conducted on NVIDIA A100 GPUs.
Training a single cMoLLM model ($N=4$) for the reported experiments requires approximately 8 A100 GPU-days.
Exact hardware specifications and software versions (PyTorch, CUDA, etc.) will be documented in the code repository.

Language modeling metrics (loss, perplexity) are computed on validation splits.
Downstream task evaluation follows standard protocols: GLUE use development set accuracy; SQuAD uses F1 score.
All experimental results use \textbf{3 random seeds} and are reported as \textbf{mean $\pm$ standard deviation (std)} throughout the paper (main tables, scaling tables, and appendix).
A plain language summary is provided in \cref{app:summary}.

We compare against a dense GPT-2 baseline (standard Transformer), ParaScale~\citep{chen2025parallel} (reimplemented with virtual token streams), and AltUp~\citep{baykal2023alternating} (reimplemented with auxiliary prediction branch).
All baselines use identical training data, hyperparameters where applicable, and evaluation protocols to ensure fair comparison.
Reproducibility details are provided in \cref{app:reproducibility}.

\section{Use of LLM}

The authors used generative AI tools (Grammarly, ChatGPT) only for grammar checking and language polishing. All technical content, experimental design, data analysis, and conclusions were generated and verified by the human authors. The use of AI tools does not affect the originality or authorship of this work.
\end{document}